\documentclass{article} 
\usepackage{iclr2027_conference}
\usepackage{times}
\setcitestyle{numbers,square}
\usepackage{macros}
\usepackage{hyperref}
\usepackage{url}
\usepackage[toc,page,header]{appendix}

\makeatletter
\renewcommand{\@maketitle}{%
  \vbox{\hsize\textwidth
    {\LARGE\sc\@title\par}
    \vskip 0.3in
    \begin{center}
    
\aut{Grégoire Dhimoïla}{\afn{a,b}}\thanks{\textit{Correspondence to:} \textcolor{secondary}{gregoire.dhimoila@ens-paris-saclay.fr}\\ \indent ~~\textit{Source code:} \href{https://github.com/Parabrele/CoSAE}{\raisebox{-2.8pt}{\includegraphics[height=10pt]{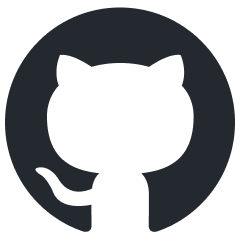}}
\texttt{\textcolor{secondary}{https://github.com/Parabrele/CoSAE}}}} \quad
\aut{Agustin Picard}\afn{c} \quad
\aut{Victor Boutin}{\afn{d}} \quad 
\aut{Thomas Fel}{\afn{e}} \quad 
\aut{Thomas Serre}{\afn{a}} 
\vspace{0mm}\\
\afn{a}\af{Brown University} \quad 
\afn{b}\af{ENS Paris Saclay} \quad 
\afn{c}\af{DEEL - IRT Saint Exupéry} \quad
\afn{d}\af{CNRS} \quad  
\afn{e}\af{Goodfire}
\vspace{-4mm}

    \end{center}
    \vskip 0.3in
  }%
}
\makeatother

\usepackage{minitoc}

\title{A Unifying Framework for Concept-Based Representational Similarity}

\newcommand{\aut}[1]{\textbf{#1}}
\newcommand{\af}[1]{{\small #1}}

\newcommand{\afn}[1]{\textcolor{primary}{$^{#1}$}}

\author{
\aut{Grégoire Dhimoïla}{\afn{a,b}}\thanks{\textit{Correspondence to:} \textcolor{secondary}{gregoire.dhimoila@ens-paris-saclay.fr}\\ \indent ~~\textit{Source code:} \href{https://github.com/Parabrele/CoSAE}{\raisebox{-2.8pt}{\includegraphics[height=10pt]{figures/github-mark.png}}
\texttt{\textcolor{secondary}{https://github.com/Parabrele/CoSAE}}}} \quad
\aut{Victor Boutin}{\afn{c}} \quad 
\aut{Agustin Picard}\afn{d} \quad
\aut{Thomas Fel}{\afn{e}} \quad 
\aut{Thomas Serre}{\afn{a}} 
\vspace{0mm}\\
\afn{a}\af{Brown University} \quad 
\afn{b}\af{ENS Paris Saclay} \quad 
\afn{c}\af{CNRS} \quad  
\afn{d}\af{DEEL - IRT Saint Exupéry} \quad
\afn{e}\af{Goodfire}
\vspace{-4mm}
}

\begin{document}

\doparttoc %
\faketableofcontents %

\maketitle

\begin{abstract}
Learned representations across models and modalities often exhibit striking structural similarities, suggesting shared underlying concept decompositions. However, concept alignment remains poorly defined: existing approaches optimize different objectives under the same terminology, obscuring what is actually aligned.
We propose a unifying framework that decomposes alignment along two axes: \emph{what} is aligned (representations vs.\ concepts) and \emph{at what level} (instance-wise vs.\ distributional). 
This induces four corresponding properties---instance-wise and distributional variants of translation and concept consistency---and reveals precisely which of these guarantees existing methods provide.
We further introduce \InterVenchA, an intervention-based benchmark that separately measures extraction quality, translation quality, and concept consistency---which we validate against human perception of concept alignment.
Using this framework, we show that commonly assumed equivalences between alignment objectives fail in practice: optimizing one property does not reliably recover the others, and, on real embeddings, purely unsupervised objectives fail to recover meaningful instance-level alignment.
As a constructive consequence, a model, the Coupled SAE (CoSAE), matches or exceeds specialized alignment methods such as crosscoders and Universal SAEs on their own target metrics, by jointly enforcing the properties identified by our framework without any architectural innovation.
Moreover, as little as 1\% paired data is sufficient to recover instance-level alignment when anchoring distributional objectives.
Overall, concept alignment is fundamentally multi-objective, and progress requires treating it as such. Our framework and benchmark provide the means to define, measure, and compare concept alignment methods on this basis.
\end{abstract}

\section{Introduction}

\paragraph{}A recurring finding in modern deep learning is that internal representations learned by different models are often far less arbitrary than their parameterizations suggest. 
Across architectures, random initializations, and training objectives, independently trained networks frequently exhibit substantial geometric similarity~\citep{kornblith2019similarity,bansal2021revisiting}.
This phenomenon extends across modalities: models such as CLIP~\citep{radford2021}, LiT~\cite{zhai2022lit}, and SigLIP~\cite{zhai2023sigmoid} learn representations that align modalities strongly enough to support zero-shot transfer.
Together, these results suggest that learned representations share non-trivial structure across systems, including across modalities.
\vspace{-2mm}

\paragraph{}If learned representations do share non-trivial structure, the key question is how to characterize it.
Representational similarity metrics (e.g., CCA~\citep{hoteling1936cca}, CKA~\cite{kornblith2019similarity}, and variants~\cite{gretton2007kernel,raghu2017svcca,ding2021grounding}) provide a coarse geometric view of representational overlap. They do not, however, identify which features are shared across systems, nor whether those features admit a common interpretable decomposition.
Sparse autoencoders (SAEs)~\cite{cunningham2023sparse,bricken2023monosemanticity,fel2023holistic} provide one such decomposition by expressing neural activations in terms of sparse, interpretable features---\emph{concepts}---and recent work has begun extending this approach across models and modalities~\cite{lindsey2024sparse,thasarathan2025universal,dhimoila2026cross}.
This line of work naturally raises the question of \emph{concept alignment}, which we use informally to mean that different systems capture the same underlying structure in their representations. In practice, however, this notion splits along two distinct axes. The first asks whether two systems' latent spaces share aligned geometries, such that one's representations can be mapped to the other's through their respective \emph{concept} dictionaries. We refer to this as \emph{translation}.
The second asks for something stronger: that the underlying concepts themselves match one-to-one across systems. We refer to this as \emph{concept consistency}.
Existing methods, however, rarely make clear which of these two properties they target: different methods optimize different objectives under the same terminology (see~\cite{klabunde2024survey} for a survey), leaving it unclear which properties are enforced, and when apparent alignment reflects genuinely shared structure.
\vspace{-2mm}

\begin{figure}[t]
    \centering
    \includegraphics[width=0.99\textwidth]{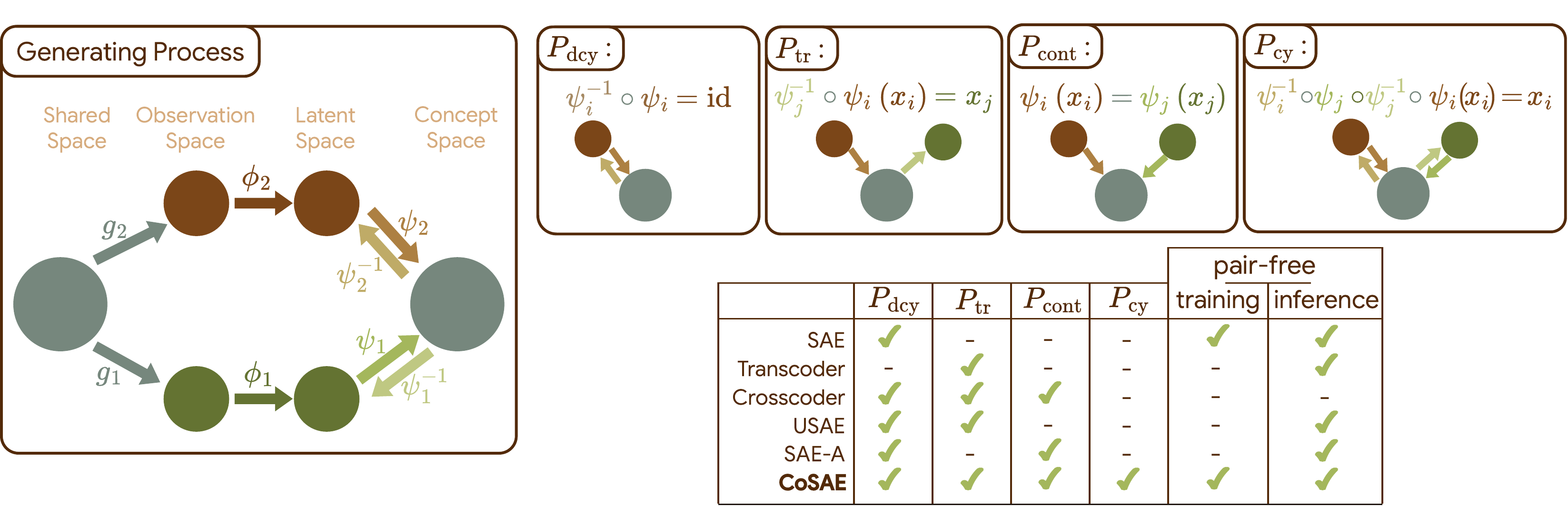}
    \caption{Summary of the proposed framework. \textbf{Left:} Illustration of a generative process where observations stem from a shared space through modality-specific generators $g_i$. Feature extractors ($\phi_i$) learn to invert this generative process up to some transform $\psi_i$. \textbf{Top~right:} Different properties of $\psi_i$ that achieve different types of alignment. \textbf{Bottom~right:} Summary of previous SAE-based methods for concept extraction, along with which property they are trained to satisfy. \textit{Pair-free} indicates whether instance-wise pairs of inputs are required for the alignment during training and inference of each SAE.}
    \label{fig:framework}
\end{figure}

\paragraph{}Any alignment objective must specify both \emph{what} is being aligned---representations (through translation) or concepts---and \emph{at what level} alignment is required---for individual instances or in distribution. In this paper, we argue that these distinctions define a framework for concept alignment.
Existing methods instantiate different parts of this space. Our framework makes these choices explicit, providing a rigorous language for comparing existing methods---from standard SAEs~\citep{cunningham2023sparse} and transcoders~\citep{dunefsky2024transcoders} to crosscoders~\citep{lindsey2024sparse}, Universal SAEs (USAEs)~\citep{thasarathan2025universal}, and SAE-Aligned (SAE-A)~\citep{dhimoila2026cross} (\Cref{fig:framework}).%

Building on this framework, we make the following contributions:
\vspace{-2mm}
\begin{itemize}[leftmargin=5mm]
    \item \textbf{A unifying framework for concept alignment.} We formalize concept alignment with four distinct properties---translation, concept consistency, instance-wise and distributional alignment---and use this decomposition to contextualize existing methods in a common design space.\vspace{-1mm}
    
    \item \textbf{A diagnostic benchmark for concept alignment.} We introduce \InterVenchA, which naturally extends SAEBench~\citep{karvonen2025saebench} evaluation to independently measure extraction quality, translation quality, and concept consistency.\vspace{-1mm}

    \item \textbf{Human validation of the benchmark.} We show through a controlled human study that our alignment score agrees with human perception of concept alignment, confirming that it captures a meaningful notion of alignment rather than an arbitrary proxy.

    \item \textbf{An empirical audit of alignment objectives.} On both synthetic and real embeddings, we show that several commonly assumed relationships between alignment properties do not hold in practice: translation does not reliably yield concept consistency (and vice versa)%
, and, on real embeddings, distributional matching does not guarantee instance-level alignment.\vspace{-1mm}

    \item \textbf{Correct diagnosis beats architectural novelty.} We introduce the Coupled SAE (CoSAE), which combines only well-known regularization terms yet matches or exceeds specialized alignment methods such as crosscoders~\citep{lindsey2024sparse} and Universal SAEs (USAEs)~\citep{thasarathan2025universal} on their own target metrics despite using scarce supervision, requiring as little as 1\% of the data to be paired.

\end{itemize}
\vspace{-1mm}

Taken together, these findings point to a simple picture: concept alignment is a multi-dimensional problem; common surrogate objectives only capture part of it; and the strongest practical performance arises from combining complementary regularizers with minimal anchoring supervision. Beyond the specific SAE instantiation studied here, we hope this framework provides a clearer basis for defining, evaluating, and optimizing concept alignment.

\section{Related Work}
\label{sec:RW}

\paragraph{Representation Alignment.} A large body of work has studied representation similarity, aiming to characterize whether two neural representations encode similar information. Early approaches rely on statistical or geometric comparisons between representation spaces, such as Canonical Correlation Analysis (CCA)~\citep{hoteling1936cca,raghu2017svcca}, Centered Kernel Alignment (CKA)~\citep{kornblith2019similarity}, and related methods~\citep{kriegeskorte2008representational,gretton2007kernel,ding2021grounding}. These methods provide a coarse-grained view of similarity by identifying shared linear subspaces or measuring invariance under orthogonal transformations. While these approaches have been instrumental in establishing that independently trained models often converge to similar representations~\citep{ding2021grounding,huh2024platonic,koepke2026back}, they remain limited in their interpretability~\citep{ding2021grounding,groger2026revisitingplatonicrepresentationhypothesis}. In particular, they characterize the degree to which representations are similar, but not \emph{how} or \emph{why}. %
These limitations have motivated a shift from coarse, single-score metrics toward finer-grained analyses that operate at the level of individual features or concepts~\citep{lindsey2024sparse}.

\paragraph{Interpretability: Towards Concept-based Analysis.} The interpretability field aims at understanding the internal workings of neural networks. Many tools have been developed to dissect learned representations~\citep{bau2017network,olah2017feature,gilpin2018explaining}. The overall effort is currently shifting from feature visualization~\citep{olah2017feature,fel2023unlocking} and attribution methods~\citep{zeiler2013visualizing,petsiuk2018rise,sundararajan2017axiomatic} to concept-based explanations~\citep{ghorbani2019towards,zhang2021invertible,elhage2022superposition,fel2023holistic}. This shift exemplifies the trend of moving from input-based explanations to internal explanations. Concept superposition and extraction originated in sparse coding~\citep{olshausen1996emergence,olshausen1997sparse,olshausen2004sparse,lee2006efficient,rentzeperis2023beyond}, and are grounded in compressed sensing and random projection theory~\citep{johnson1984extensions,larsen2017optimality}. It is therefore typically framed as a dictionary learning problem~\citep{rubinstein2010dictionaries,elad2010sparse,tovsic2011dictionary,dumitrescu2018dictionary}: Sparse autoencoders (SAEs) have now become the default approach for concept extraction~\citep{gao2024scaling,rajamanoharan2024jumping,bussmann2024batchtopk,fel2025archetypal,costa2025flat} following their empirical success in interpretability for both vision~\citep{gorton2024missing,fel2025archetypal,fel2025into} and language modeling~\citep{cunningham2023sparse,bricken2023monosemanticity,surkov2024unpacking}.

\paragraph{Concept-based Representation Alignment.} Recent work moved from global geometric similarity to concept-level similarity. Instead of comparing representations as monolithic objects, these approaches aim to decompose them into interpretable features and study how these features align across models or modalities. SAEs have emerged as a central tool in this setting, enabling the extraction of structured, often interpretable feature dictionaries. Crosscoders~\citep{lindsey2024sparse} enforce a shared coding space across representations through architectural constraints, while Universal SAEs (USAEs)~\citep{thasarathan2025universal} (resp. Aligned SAEs~\citep{dhimoila2026cross}) rely on regularization through translation (resp. concept consistency) objectives to align feature spaces across models. Despite sharing a common goal, these methods differ significantly in the properties they enforce and in the assumptions they implicitly rely on. More details on how these methods fit into our framework are provided in \cref{app:extended_RW}.

\section{Framework Formulation}

\begin{wrapfigure}{r}{0.6\textwidth}
\vspace{-7mm}
    \centering
    \begin{subfigure}[h]{0.49\linewidth}
        \centering
        \includegraphics[width=0.99\linewidth]{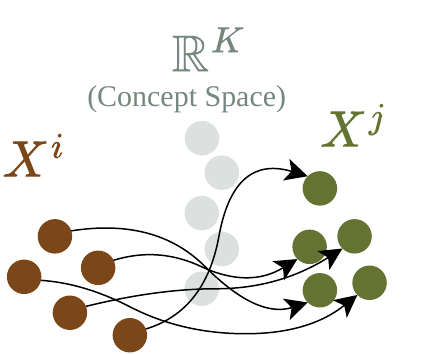}
        \caption{$\Ptr : \, (\psi^{-1}_j \circ \psi_i) (\xvec_i) = \xvec_j$}
    \end{subfigure}
    \hfill
    \begin{subfigure}[h]{0.49\linewidth}
    \centering
        \includegraphics[width=0.99\linewidth]{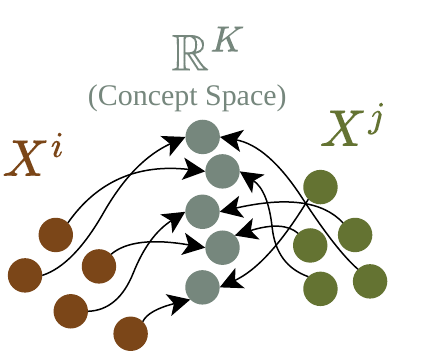}
        \caption{$\Pcont: \, \psi_i(\xvec_i) = \psi_j(\xvec_j)$}
    \end{subfigure}
    \caption{Element-wise alignment.}
    \label{fig:element_wise}
\end{wrapfigure}

\paragraph{Concept Alignment.} Let $I$ be an index set whose elements correspond to representation instances. Each $i \in I$ may, for example, index the same modality embedded in different models, different layers of a single model, checkpoints of a single model at different training steps, modality-specific instantiations of an abstract idea, etc.

Let $X$ be a set of data points equipped with a measure $\mu$, and $\phi_i : X \to \R^d$ be a representation function for each $i \in I$. Denote by $X^i \coloneq \phi_i(X)$ the set of representations of $X$ under $\phi_i$, and by $\mu_i \coloneq \phi_i \# \mu$ the pushforward measure on $X^i$ induced by $\phi_i$. Let $\psi_i : X^i \to \R^K$ be concept extraction functions with corresponding decoders $\psi_i^{-1} : \R^K \to X^i$ abusively denoted with the inverse notation for readability.
Under this setting, and with $\xvec_i = \phi_i(x)$ and $\xvec_j = \phi_j(x)$, we can define the following alignment properties of \emph{translation} and \emph{concept consistency} at the instance level (\Cref{fig:element_wise}).

\begin{wrapfigure}{l}{0.6\textwidth}
    \centering
    \begin{subfigure}[h]{0.49\linewidth}
        \centering
        \includegraphics[width=0.99\linewidth]{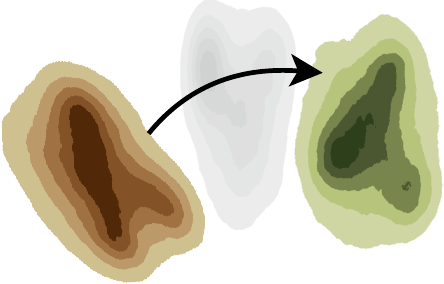}
        \caption{$\Ptrdist\!: \! (\psi^{\hspace{-1mm}-1}_j\!\circ \psi_i) \# \mu_i \!=\!\mu_j$}
    \end{subfigure}
    \hfill
    \begin{subfigure}[h]{0.49\linewidth}
    \centering
        \includegraphics[width=0.99\linewidth]{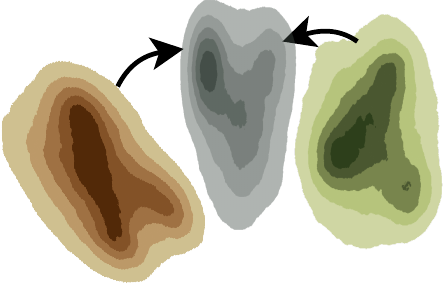}
        \caption{$\Pcontdist\!: \! \psi_i \# \mu_i = \psi_j \# \mu_j$}
    \end{subfigure}
    \caption{Distributional alignment.}
    \label{fig:distributional}
\end{wrapfigure}

\paragraph{}If the $X^i$ are such that instance-wise correspondences are not available, i.e., we do not have access to the underlying $X$, we can translate these into distributional alignment properties as follows.
Let $(\psi^{-1}_j \circ \psi_i) \# \mu_i$ define a measure on $X^j$, and $\psi_i \# \mu_i$ define a measure on the concept space $\R^K$. We define the distributional alignment properties corresponding to translation and concept consistency in \Cref{fig:distributional}.

\begin{wrapfigure}{r}{0.7\textwidth}
\vspace{-7mm}
    \centering
    \begin{subfigure}[h]{0.49\linewidth}
        \centering
        \includegraphics[width=0.8\linewidth]{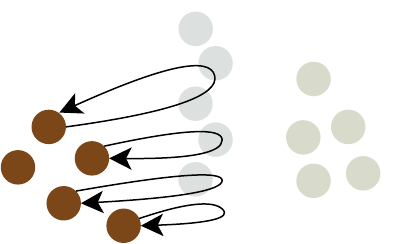}
        \caption{$\Pdcy: \, (\psi^{-1}_i \circ \psi_i) (\xvec_i) = \xvec_i$}
    \end{subfigure}
    \hfill
    \begin{subfigure}[h]{0.49\linewidth}
    \centering
        \includegraphics[width=0.8\linewidth]{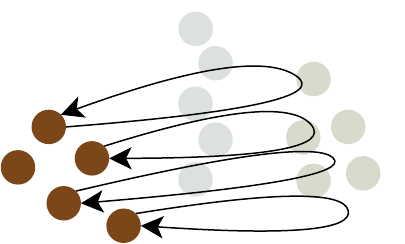}
        \caption{$\Pcy\!: \! (\psi^{\!\!-1}_i\hspace{-0.0mm}\circ\hspace{-0.0mm}\psi_j\hspace{-0.0mm}\circ\hspace{-0.0mm}\psi_j^{\!\!-1}\hspace{-0.0mm}\circ\hspace{-0.0mm}\psi_i) (\xvec_i) = \xvec_i$}
    \end{subfigure}
    \caption{Self reconstruction: demi and full cycle.}
    \label{fig:reconstruction}
\end{wrapfigure}

\paragraph{}In addition to these alignment properties, we can also define standard autoencoding properties, effectively stating that $\psi^{-1}_i$ is indeed the inverse of $\psi_i$ on the support of $\mu_i$.
The \emph{demi-cycle} property $\Pdcy$ is the standard reconstruction objective of autoencoders, while the full \emph{cycle} $\Pcy$ is an alternative way to couple autoencoders without requiring instance-wise correspondences through reconstruction only. These properties are defined in \Cref{fig:reconstruction}.

\subsection{Alignment Dualities}
\label{sec:dualities}

\begin{proposition}
\label{prop:PtrPcont}
Suppose injectivity of the decoders, and suppose $\Pdcy$. Then
$\Ptr \Leftrightarrow\Pcont$
\end{proposition}

\vspace{-5mm}
\begin{proof}
    See \cref{app:proofs} or~\citep{devillers2024semi}.
\end{proof}

\paragraph{}
Note that in the framework of SAEs, the injectivity assumption is directly violated by the overcompleteness of the dictionary and the linearity of the decoder. However, the sparsity assumption goes in the other direction, as injectivity is restricted on the sparse support of the concept space. For this reason, we empirically check whether this equivalence holds (\cref{sec:dualities_in_practice}). It does not.
A second asymmetry arises between instance-wise and distributional alignment:

\vspace{3mm}
\begin{proposition}
\label{prop:PtrPtrdist} $\Ptr \Rightarrow \Ptrdist \quad \land \quad \Pcont \Rightarrow \Pcontdist$
\end{proposition}

\vspace{-5mm}
\begin{proof}
    See \cref{app:proofs}.
\end{proof}
\vspace{-2mm}

\paragraph{}
In other words, the instance-wise properties imply their distributional counterparts (\Cref{prop:PtrPtrdist}), but the converse generally fails without strong assumptions on the distributions and function class. In \cref{sec:anchoring}, we show that a few anchor pairs can partially restore this bridge in practice.

\paragraph{}
Finally, while cycle consistency has been used as a surrogate in pair-free translation settings~\citep{zhu2017unpaired,artetxe2017unsupervised}, we show in \cref{sec:dualities_in_practice} that it is not a reliable proxy for alignment in our setting.  
Recent works on joint concept recovery focus only on subsets of these properties, with little to no discussion of the others and their interactions~\citep{lindsey2024sparse,thasarathan2025universal,dhimoila2026cross}. This motivates the need for a unifying framework to understand the properties and assumptions underlying the different methods used in the literature, and to design new methods that better enforce these properties for future concept alignment studies.

\paragraph{Additional assumptions.} In the SAE framework, this work relies on several other assumptions described here and expanded in \cref{app:prh_lrh}. First, we assume that representations follow the linear representation hypothesis, and that latent space is the projection of a compressed, sparse and higher dimensional space called the concept space. These are the assumptions underlying compressed sensing, and are standard in the literature. Then, we assume a generative form of the platonic representation hypothesis (PRH), stating that independently trained models, across different domains, learn to recover the same concept space.

\paragraph{}
A complementary controlled experiment with known shared and private factors (\cref{app:synthetic_sharedness}) confirms this picture: instance-wise translation regularization recovers shared structure without erasing private, domain-specific factors.

\subsection{Operationalizing the framework}
\label{sec:operationalizing}

\vspace{-2mm}
\paragraph{Regularization.} Each of the properties defined above can be operationalized as regularization terms in the training loss:
\vspace{-1mm}
\begin{align}
\label{eq:loss}
\mathcal{L}_{\mathrm{SAE}} = \dcycolor{\alpha_{\mathrm{dcy}}} \Ldcy + \cycolor{\alpha_{\mathrm{cy}}} \Lcy + \trcolor{\alpha_{\mathrm{tr}}} \Ltr + \trdistcolor{\alpha^{\mathrm{dist}}_{\mathrm{tr}}} \Ltrdist + \contcolor{\alpha_{\mathrm{cont}}} \Lcont + \contdistcolor{\alpha^{\mathrm{dist}}_{\mathrm{cont}}} \Lcontdist
\vspace{-4mm}
\end{align}

\paragraph{Linear Models.} In the linear regime, where both the encoders and decoders are linear maps, the different objectives and their combinations collapse to a few analytical solutions. \textbf{(A)} Without coupling, i.e., with only reconstruction, we recover PCA solutions for each domain independently (\Cref{lem:pca}). \textbf{(B)} Introducing coupling through concept consistency (\Cref{lem:cca,lem:cont_dcy}) or translation (\Cref{lem:tr_rrr,lem:tr_dcy}) both lead to the same solution of the CCA for the encoders, while decoders are some variation of optimal linear readouts of the CCA projections. This perfectly illustrates the duality between $\Ptr$ and $\Pcont$ in the linear case. \textbf{(C)} Distributional alignment fails to introduce any meaningful coupling, and only enforces covariance matching (\cref{app:linear_dist}). \textbf{(D)} Finally, although cycle consistency appears to introduce coupling, it actually fails to do so in the linear case, and only enforces PCA solutions for each domain independently (\cref{app:linear_cy}).

\vspace{-2mm}
\paragraph{SAE.} As discussed in \cref{sec:RW}, over the recent years, the SAE framework for sparse coding has become the default concept extraction in the community. We adopt the \emph{batchtopk} architecture throughout this work, with details in \cref{app:training_details}.

\vspace{-2mm}
\paragraph{Reconstruction.} As is standard practice~\citep{gao2024scaling}, the demi-cycle consistency property $\Pdcy$ can simply be operationalized as an $\ell_2$ reconstruction loss (\Cref{eq:reconstruction_loss}). The translation $\Ptr$ and cycle properties $\Pcy$ can also be operationalized as an $\ell_2$ loss on the translated representations (\Cref{eq:translation_loss,eq:cycle_loss}).
\begin{align}
\label{eq:reconstruction_loss}
\Ldcy &\coloneqq \E_{\xvec_i \sim \mu_i} \| (\psi^{-1}_i \circ \psi_i) (\xvec_i) - \xvec_i \|^2\\
\label{eq:translation_loss}
\Ltr &\coloneqq \E_{\xvec \sim \mu} \| (\psi^{-1}_j \circ \psi_i) (\xvec_i) - \xvec_j \|^2\\
\label{eq:cycle_loss}
\Lcy &\coloneqq \E_{\xvec_i \sim \mu_i} \| (\psi^{-1}_i \circ \psi_j \circ \psi_j^{-1} \circ \psi_i) (\xvec_i) - \xvec_i \|^2
\vspace{-6mm}
\end{align}
\paragraph{Concept Consistency.} Crosscoders \citep{lindsey2024sparse} enforce $\Pcont$ not through regularization of disjoint encoders, but through summation of all encoders' outputs to create a common code: $\psi_i'(\xvec_i) = \sum_{j \in I} \psi_j(\xvec_j)$. This is a very strong constraint, since interpretability is imposed only on the aggregated code rather than on the individual encoder outputs. The duality between $\Ptr$ and $\Pcont$ is here afforded for free by the fact that the code of $\xvec_i$ already contains all information about $\xvec_j$. We show in \cref{sec:recipe} that the standalone $\psi_i$ are not able to perform translation in general, and do not even satisfy $\Pcont$. %
To resolve this tension, we replace the summation in $\psi_i'$ with a regularization term for $\Pcont$. %
We simply use an $\ell_2$ loss on the difference between the codes (\Cref{eq:contrastive_loss}).
\begin{align}
\label{eq:contrastive_loss}
\Lcont \coloneqq \E_{x \sim \mu} \left\| \psi_i(\xvec_i) - \psi_j(\xvec_j) \right\|^2
\vspace{-2mm}
\end{align}

\paragraph{Distributional alignment.}
When instance-wise correspondences are unavailable, enforcing $\Ptrdist$ or $\Pcontdist$ requires matching distributions rather than individual samples. A natural approach is to minimize a discrepancy between pushforward measures, for instance, using optimal transport. However, distributional distances like the Wasserstein distance are both computationally expensive and poorly behaved in high dimensions and mini-batch settings \citep{fournier2015rate,peyre2018computational}, making them ill-suited for large-scale SAE training.

A practical alternative is provided by the Cramer-Wold theorem, which states that two distributions are equal if and only if their one-dimensional projections along all directions are equal. We take inspiration from \citet{balestriero2025lejepa} and use empirical characteristic function (ECF) based tests for our distributional losses, as defined in \cref{app:L_dist}. These losses are equivalent to sliced Maximum Mean Discrepancies (MMDs):
\begin{equation}
    \Ltrdist \coloneqq \mathrm{MMD}\left((\psi^{\hspace{-1mm}-1}_j\!\circ \psi_i) \# \mu_i, \mu_j\right) \quad \Lcontdist = \mathrm{MMD}\left(\psi_i \# \mu_i, \psi_j \# \mu_j\right)
\end{equation}

\section{Framework Validation}
\vspace{-2mm}

\subsection{Experimental Protocol}
\label{sec:protocol}
\vspace{-2mm}

\subsubsection{Metrics}
\vspace{-2mm}

\paragraph{\InterVench for Concept Extraction.} For SAE quality in real embeddings, we do not have access to ground truth concepts. Therefore, we rely on standard metrics, derived from SAEBench \citep{karvonen2025saebench}, as a proxy for concept extraction quality. SAEBench is a battery of tests designed to evaluate the quality of concepts extracted by SAEs in the context of language models. We adapt the relevant metrics to our setting and use them to evaluate the quality of the concepts extracted by our CoSAE under different regularization regimes.

Specifically, we focus on reconstruction and intervention-based metrics from SAEBench, which we further extend and factorize into the following classes.
First, \InterVenchLatent metrics evaluate the quality of the dictionary to characterize latent geometry. For example, probes are parameterized as sparse combinations of dictionary atoms. An overall high score on these metrics means that \textbf{the dictionary} is both \textbf{a good descriptor of latent geometry} and \textbf{an actionable tool} for downstream task uses.
Second, \InterVenchConcept metrics evaluate \textbf{the quality of the concept space as its own representation space.}
We report a single aggregate \emph{extraction score} $\Mextract$. More details in \cref{app:InterVench}.
\vspace{-1mm}

\paragraph{\InterVenchA for Concept Alignment.} We extend \InterVench from concept extraction to concept alignment by turning self-reconstruction metrics into transfer metrics. Concretely, we replace self-reconstruction with cross-domain translation, and reinterpret intervention-based metrics as tests of whether concepts learned in one domain transfer meaningfully to another. This yields two complementary views of alignment: \InterVenchLatent measures translation quality independently of concept consistency (with an aggregate score $\Mtr$), while \InterVenchConcept measures concept consistency independently of translation (through the aggregate score $\Mcont$). The mean of $\Mtr$ and $\Mcont$ is denoted by the \emph{alignment score} $\Malign$.
More details can be found in \cref{app:InterVenchA}.
\vspace{-1mm}

\paragraph{Human Validation.} Because $\Malign$ is built from proxy metrics rather than ground truth concepts, we validate it directly against human judgment. To measure whether this benchmark tracks a human-perceptible notion of alignment, we ran a human evaluation study (\cref{app:human_validation}) in which annotators are shown a reference set of maximally activating images for a given atom and must pick, among two candidate sets, the one from the corresponding atom in another domain. We compare three cells: a \emph{control} (no alignment regularization), an \emph{aligned bad} cell (atoms with high $\Mextract$ but low $\Malign$), and an \emph{aligned good} cell (high $\Mextract$ and high $\Malign$). A crossed random-effects binomial mixed model shows that participants are significantly more likely to identify the correct match in the aligned good cell than in either the aligned bad or control cells (odds ratios $1.50$ and $2.75$ respectively, both with $95\%$ credible intervals excluding $1$), confirming that $\Malign$ agrees with human judgment of alignment quality.
\vspace{-1mm}

\subsubsection{Datasets and Models}
\label{sec:datasets_models}

\paragraph{Synthetic DGPs.} 
We use two complementary synthetic settings. The first follows precisely a generative formulation of the platonic representation hypothesis (\cref{hyp:generative_process,hyp:sparse_inversion}) motivating our framework, and evaluates recovery under controlled transformations of sparse latent factors
(\cref{hyp:generative_process,hyp:sparse_inversion,app:synthetic_DGP}). 
The second uses a separate linear generator with fully shared, partially shared, or entirely private factors. It distinguishes within-view factor recovery from cross-view correspondence and evaluates sharedness using ground-truth matching, permutation-calibrated dependence tests, and intervention controls. Its protocol and results are reported in
\cref{app:synthetic_sharedness}.
\vspace{-1mm}

\paragraph{Vision Encoders.} We conduct a battery of cross-model alignment experiments on vision models. We select a ViT ({\verb "google/vit-base-patch16-224"}) \citep{dosovitskiy2021}, DinoV2 ({\verb "facebook/dinov2-base"}) \citep{oquab2023dinov2}, and SigLIP ({\verb "google/siglip-base-patch16-224"}) \citep{zhai2023sigmoid} from Hugging Face \citep{wolf2019huggingface}. The choice of these models was primarily to ensure a clean comparison with \citet{thasarathan2025universal}. These models are used to extract embeddings for the ImageNet training and validation splits. SAEs are trained on these embeddings (details in \cref{app:training_details}). For evaluation of both concept extraction and alignment quality, we rely on metrics aggregated in \InterVenchA.

\paragraph{LLMs and VLMs.} We also run similar experiments on tokens extracted from Large Language Models. We select Llamas~\citep{touvron2023llama} and Gemmas~\citep{team2024gemma}, and extract tokens at the 75-th percent layer as is common practice. For token alignment across different tokenizers, we adopt the same procedure as~\citep{jiralerspong2026cross}. Finally, we also conduct cross-modal alignment experiments on Vision Language Models. More details in~\cref{app:experiment_detail}. Only the vision encoder's settings results are reported in the main body of this paper. See \cref{app:results_ablation} for the other settings.

\subsection{Alignment Duality in Practice}
\label{sec:dualities_in_practice}

\begin{table*}[t]
  \centering
  \resizebox{\linewidth}{!}{
  \begin{tabular}{lCCCCCcCCCc}
\toprule
& \multicolumn{10}{c}{\InterVenchA --- Vision} \\
\cmidrule(lr){2-11}
& \multicolumn{5}{c}{Translation} & \multicolumn{1}{c}{} & \multicolumn{3}{c}{Concept Consistency} & \multicolumn{1}{c}{} \\
\cmidrule(lr){2-6} \cmidrule(lr){8-10}
Regularization
& $R^2$ & Faith & SP & Unlearn & TPP & $\Mtr$ &
                  SP & Unlearn & TPP & $\Mcont$
\\
\midrule
    \textit{(a)}~~$\Ldcy$              & 0.000 & 0.686 & 0.000 & 0.000 & 0.000  & \trcolor{0.137} & 0.000 & 0.000 & 0.000  & \contcolor{0.000} \\
    \midrule
    \textit{(b)}~~$\Ldcy + \Ltr$      & 0.466 & 1.185 & 0.535 & 0.880 & 0.575  & \trcolor{0.728} & 0.190 & 0.241 & 0.221  & \contcolor{0.217} \\
    \textit{(c)}~~$\Ldcy + \Ltrdist$   & 0.000 & 0.781 & 0.295 & 0.812 & 0.733  & \trcolor{0.524} & 0.000 & 0.000 & 0.000  & \contcolor{0.000} \\
    \midrule
    \textit{(d)}~~$\Ldcy + \Lcont$     & 0.000 & 0.700 & 0.000 & 0.000 & 0.000  & \trcolor{0.140} & 0.460 & 0.826 & 0.521  & \contcolor{0.602} \\
    \textit{(e)}~~$\Ldcy + \Lcontdist$ & 0.000 & 0.701 & 0.000 & 0.000 & 0.000  & \trcolor{0.140} & 0.000 & 0.000 & 0.000  & \contcolor{0.000} \\
    \midrule
    \textit{(f)}~~$\Ldcy + \Ltr + \Lcont$       & 0.471 & 1.182 & 0.545 & 0.924 & 0.544  & \trcolor{0.733} & 0.469 & 0.877 & 0.510  & \contcolor{0.619} \\
    \textit{(g)}~~$\Ldcy + \Ltrdist + \Lcontdist$      & 0.000 & 0.714 & 0.000 & 0.000 & 0.000  & \trcolor{0.143} & 0.000 & 0.000 & 0.000  & \contcolor{0.000} \\
    \textit{(h)}~~mixed (1 in 10)                    & 0.524 & 1.151 & 0.522 & 0.923 & 0.523  & \trcolor{0.729} & 0.727 & 0.935 & 0.779  & \contcolor{0.813} \\
    \textit{(i)}~~mixed (1 in 100)                    & 0.422 & 1.093 & 0.627 & 0.885 & 0.632  & \trcolor{0.732} & 0.715 & 0.931 & 0.737  & \contcolor{0.795} \\
    \textit{(j)}~~mixed (1 in 1000)                    & 0.410  & 0.946  & 0.467  & 0.602 & 0.591 & \trcolor{0.602} & 0.353 & 0.821 & 0.458  & \contcolor{0.544} \\
    \bottomrule
  \end{tabular}}
  \caption{Detailed \InterVenchA metrics on \textbf{vision embeddings} under different regularization regimes. Similar tables for \InterVench and for the language and multimodal settings can be found in \cref{app:results_ablation}.}
  \label{tab:intervencha_ablation_studies_vision}
\end{table*}

\paragraph{} We train a battery of CoSAEs with different regularizations to isolate the contribution of each and measure whether alignment dualities (\cref{sec:dualities}) hold empirically. \cref{tab:intervencha_ablation_studies_vision} shows the detailed results for \InterVenchA in the vision setting. Due to our hyperparameter selection protocol (\cref{app:training_details}), concept extraction quality remains comparably high across the board. Details are reported in \cref{app:results_ablation}, along with details for other modalities (LLMs and VLMs).

\paragraph{Translation vs Concept Consistency.} We first ask whether regularization for one of translation or concept consistency is sufficient to get the other for free, or whether we need to explicitly regularize both. Neither direction of the duality holds in practice: optimizing translation improves translation, and optimizing concept consistency improves concept consistency. But neither reliably recovers the other (see \Cref{tab:intervencha_ablation_studies_vision}, lines \textit{(b)} and \textit{(d)}). The practical implication is clear: both properties must be explicitly regularized.

\paragraph{On the Effect of Cycle Consistency.} The second duality we investigate is that of alignment and cycle consistency. As mentioned above, \citep{zhu2017unpaired,artetxe2017unsupervised} use cycle consistency as a surrogate for translation in NLP, enabling pair-free training. We ask whether cycle consistency is a faithful surrogate for alignment in our framework. Training an additional SAE with regularizations $\Ldcy$ and $\Lcy$ leads to a model with $\Mtr = 0.072$ and $\Mcont = 0.000$. This model therefore fails to recover concept alignment, suggesting that the assumptions underlying the duality are not satisfied in practice.

We additionally measure the $R^2$ associated to the cycle consistency property (\cref{app:InterVench}), denoted by $R^2_{\mathrm{cy}}$. All measured models have values between $R^2_{\mathrm{cy}} \in [0.4, 0.6]$, with the exception of \textit{(b)}~$\Ldcy + \Ltr$, for which $R^2_{\mathrm{cy}} = 0.121$.
The takeaway is blunt: in our setting, cycle consistency is not a faithful surrogate for alignment. It can be high even when alignment is poor---and low despite strong alignment.
Therefore, cycle consistency should not be used as a surrogate for alignment in our framework.

\paragraph{Instance-wise vs Distributional Alignment.} Finally, we ask whether distributional objectives behave as faithful surrogates for instance-wise alignment. 
(\emph{i}) In all settings, we find that, unsurprisingly, instance-wise regularization (\cref{tab:intervencha_ablation_studies_vision} lines \textit{(b)} and \textit{(d)}) leads to both instance-wise and distributional alignment~(\cref{app:results_ablation}).
(\emph{ii}) However, the converse is false: distributional objectives do not recover instance-wise alignment. This is exactly the failure mode expected from the unidentifiability of the transport problem, and it explains why purely distributional alignment is insufficient in practice. Results in~\Cref{tab:intervencha_ablation_studies_vision} lines \textit{(c)} and \textit{(e)}, and in \cref{app:results_ablation}. We show in the next section that this does not mean that distributional objectives are entirely useless, particularly in the case of scarce supervision.

\subsection{Anchoring Alignment with Scarce Supervision}
\label{sec:anchoring}

\begin{wrapfigure}{r}{0.45\textwidth}
\vspace{-5mm}
    \centering
    \includegraphics[width=0.99\linewidth]{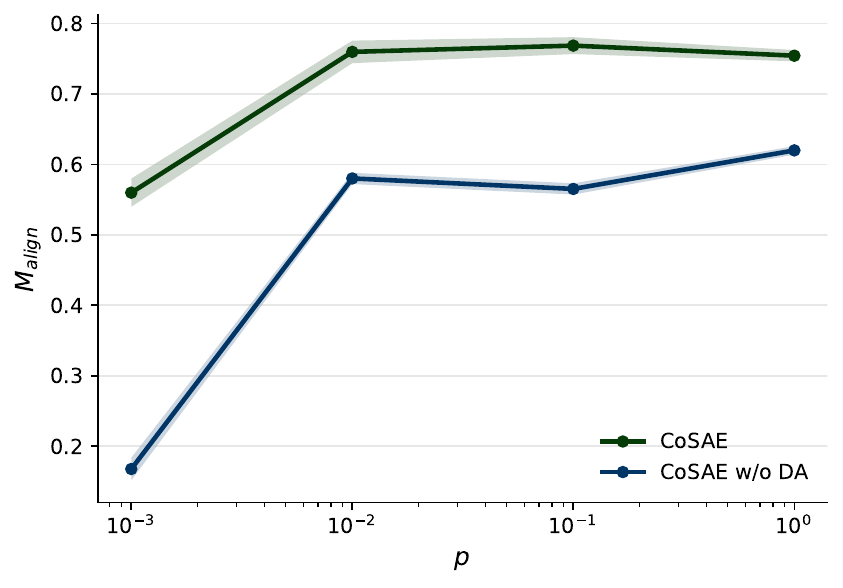}
    \caption{\textbf{(i)} High alignment score can be recovered with scarce supervision. \textbf{(ii)} CoSAEs trained with distributional alignment (in green) recover alignment scores of full supervision, contrary to those trained without (in blue).}
    \label{fig:u2vsu3_main}
\end{wrapfigure}

\paragraph{}A common issue with instance-wise regularization in multimodal settings is the scarcity of high-quality matching pairs. Therefore, it is desirable to rely on distributional regularization for the bulk of the training.
However, results from the previous experiments suggest that translation and concept consistency regularizations are complementary, and that distributional objectives are not faithful surrogates for instance-wise alignment. This motivates the need for a combination of regularization terms to achieve maximal alignment quality in the wild.
Anchoring distributional alignment using a few matched pairs could be enough to resolve the unidentifiability of the transport problem. In this experiment, we study the behavior of such mixed regimes. Specifically, we compute $\Ltr$ and $\Lcont$ only on a small subset of held-out matching pairs---with $1$ in $1000$ amounting to one per batch on average. The proportion of matched pairs is denoted by $p$. Refer to \cref{app:mixed} for full details.
\vspace{-4mm}

\paragraph{}
\cref{tab:intervencha_ablation_studies_vision} lines \textit{(f)} to \textit{(j)} show that a tiny amount of pairing is enough to resolve the ambiguity of purely distributional objectives. With as few as 1 in 100 matched pairs, we recover strong instance-wise alignment---matching or even exceeding full instance-wise supervision---while preserving concept extraction. Note that anchor points alone are sufficient to recover high alignment metrics, with the exact effect of distributional objectives being shown in \Cref{fig:u2vsu3_main,fig:u2vsu3}.

\vspace{-1mm}
\subsection{The Alignment Recipe}
\label{sec:recipe}
\vspace{-1mm}

\begin{table*}[h]
  \centering
  \vspace{-2mm}
  \scalebox{0.9}{\begin{tabular}{lcccccc}
    \toprule
    Method & Eval & $\Mextract$ ($\uparrow$) & $\Mtr$ ($\uparrow$) & $\Mcont$ ($\uparrow$) \\
    \midrule
    Crosscoder  & shared $z$     & 0.887  & 0.596  & \secondbest{0.794} \\
    Crosscoder  & standalone $z$ & 0.860  & 0.319  & 0.000 \\
    USAE        & --             & \best{0.895}  & \secondbest{0.731}  & 0.435 \\
    Ours (CoSAE)
    & --             & \secondbest{0.893}  & \best{0.732}  & \best{0.795} \\
    \bottomrule
  \end{tabular}}
  \caption{Our method (\cref{tab:intervencha_ablation_studies_vision} line \emph{(i)}) matches or exceeds both crosscoders and USAE.}
  \label{tab:recipe}
  \vspace{-3mm}
\end{table*}

\paragraph{} We conduct a battery of ablation studies to understand the contribution of each regularization term to the overall performance of the CoSAE (\cref{tab:intervencha_ablation_studies_vision} line \emph{(i)}) on both concept extraction and alignment. We train SAEs with all combinations of regularization terms, and evaluate them on our \InterVenchA for both concept extraction and alignment quality. See \cref{app:results_ablation} for the full table of results.
We also include a comparison with the original USAE and crosscoder methods. For the crosscoder, we distinguish two evaluation settings. In the first one, we consider the unmodified crosscoder, i.e., where the codes are the sum of all encoders' standalone codes. In the second evaluation setting, the codes are no longer aggregated, and each encoder is evaluated for its capacity to produce standalone codes.

\paragraph{}
Our CoSAE matches or exceeds both the USAE and the crosscoder on their respective target metrics ($\Mtr$ and $\Mcont$, respectively). CoSAE achieves the best performance on both, despite crosscoder enforcing \emph{perfect} $\Pcont$ by design through code aggregation. Moreover, when evaluated with standalone codes, the crosscoder's performance drops sharply, including a collapse in $\Mcont$. This indicates that code aggregation can mask weak standalone encoders, and that alignment observed during aggregated evaluation does not necessarily transfer to the individual encoders.
\vspace{-1mm}

\subsection{Cross-Modal Transfer with Scarce Supervision}
\vspace{-1mm}

\begin{wrapfigure}[19]{r}{0.55\textwidth}
\vspace{-7mm}
    \centering
    \includegraphics[width=1.0\linewidth]{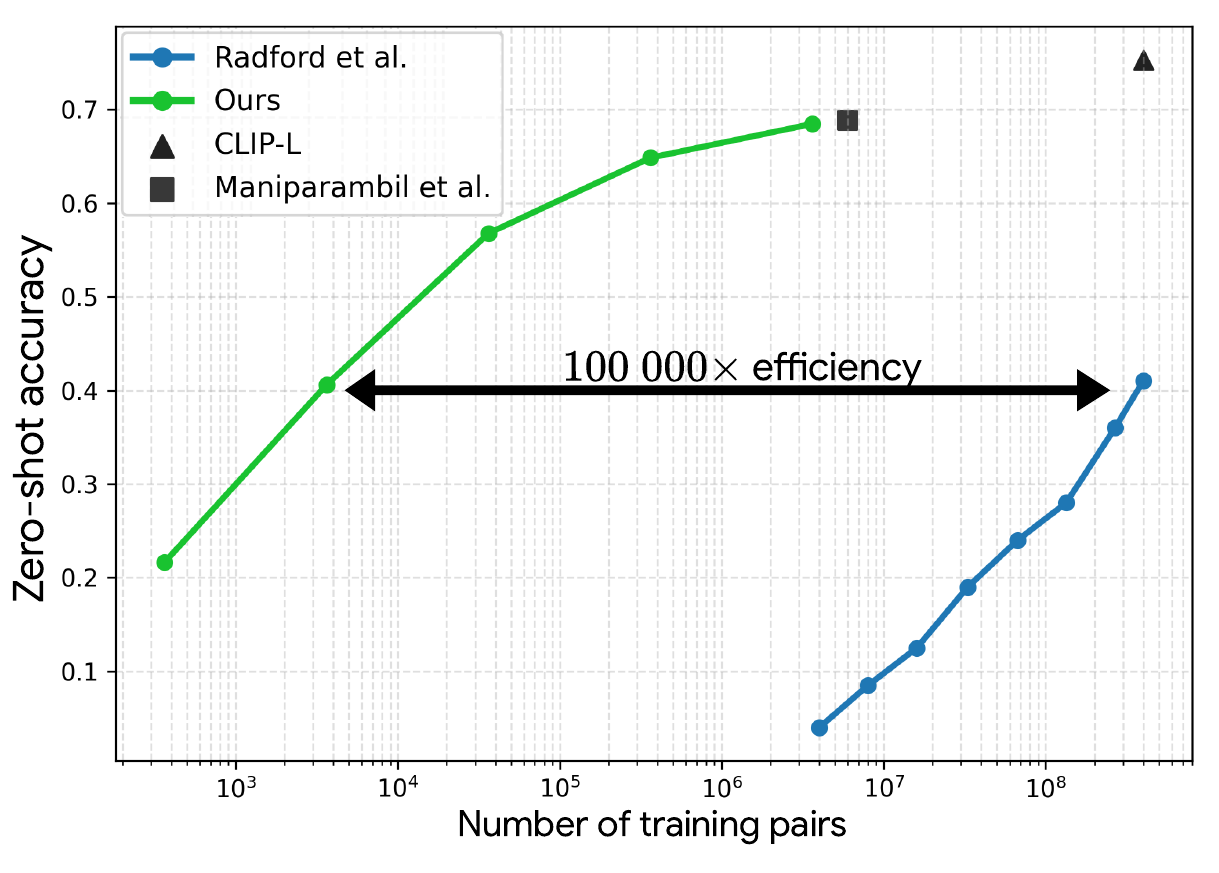}
    \caption{\textbf{Concept alignment in our CoSAE is strong enough to support competitive downstream behavior}. ImageNet zero-shot accuracy of our CoSAE trained on unimodal backbones.}
    \label{fig:cherry}
\end{wrapfigure}

\paragraph{} As a final validation, we evaluate whether the alignment learned by our CoSAE translates into competitive cross-modal transfer. We consider a setting analogous to CLIP training, replacing the vision–vision alignment of \cref{sec:datasets_models} with a vision–text setting (DINOv2 and all-roberta-large), while keeping the same training procedure. We extract CLS tokens on the subset of LAION provided by \citep{maniparambil2025harnessing}. In this regime, the translation from vision to text representations effectively acts as a learned projector, comparable to the MLP heads used by \citep{maniparambil2025harnessing}. We therefore compare these settings.

We evaluate zero-shot ImageNet accuracy obtained from the aligned representations (see  \Cref{fig:cherry}). Although CoSAE is not optimized for contrastive cross-modal training, it matches the reported zero-shot accuracy of \citep{maniparambil2025harnessing}. This suggests that the alignment learned by CoSAE is strong enough to support competitive downstream transfer, despite relying on sparse autoencoders rather than a dense projector.
\vspace{-2mm}

\section{Discussion}
\label{sec:ccl}
\vspace{-2mm}

\paragraph{}We revisit representational similarity through a concept-based lens and argue that alignment should not be treated as a single objective. Instead, it is a structured combination of properties---translation and concept consistency, each at both instance and distribution levels---and progress depends on keeping these distinctions explicit.
Empirically, widely assumed equivalences break down: translation does not imply concept consistency, cycle consistency is not a reliable surrogate, and, on real embeddings, distributional objectives alone fail to recover instance-wise alignment. Alignment is therefore inherently multi-dimensional and cannot be captured by any single objective. These claims are backed by quantitative metrics: \InterVenchA. We validate them against human judgment, confirming that $\Malign$ tracks a perceptually meaningful notion of alignment rather than an internal proxy.
\vspace{-1mm}

\paragraph{}
CoSAE is the constructive consequence of this view: rather than betting on a single surrogate, it jointly enforces complementary objectives. Strong alignment emerges only in this regime, with distributional regularization becoming effective once anchored by a few paired examples. In this setting, CoSAE matches or exceeds prior SAE-based approaches while requiring only minimal supervision.
\vspace{-1mm}

\paragraph{}
Our work also has several limitations. First, the framework focuses on shared structure and does not explicitly model idiosyncratic or domain-specific features beyond self-reconstruction. Second, our evaluation in real embeddings necessarily relies on proxy metrics rather than ground-truth concepts, so semantic alignment and functional transfer may not always coincide.
Our human study mitigates this for the alignment score specifically, but extraction quality and the underlying concepts themselves remain unverified against ground truth. This is inherent to unsupervised dictionary learning.
Third, the proposed objective combines multiple regularization terms, introducing hyperparameter sensitivity that deserves further study. Finally, our empirical validation, while broad, still covers only a limited set of architectures and datasets. Extending these findings to larger-scale and more heterogeneous settings remains an important direction for future work. 
\vspace{-1mm}

\paragraph{}
Overall, concept alignment is best understood not as a single objective, but as a coupled family of objectives. Meaningful alignment emerges not from optimizing one surrogate well, but from respecting this structure in both evaluation and method design. We hope this perspective clarifies the assumptions underlying existing methods and lays a foundation for more robust, interpretable alignment techniques in the future.

\clearpage
\subsubsection*{AI use statement}
In this work, we used generative AI tools to edit the manuscript for readability, including correcting typos and grammar. All authors reviewed the final text, and we take responsibility for the full content of this work, including any text produced with the aid of generative AI.

\subsubsection*{Acknowledgements}

This work was supported by the Office of Naval Research (N00014-24-1-2026), the National Science Foundation (IIS-2402875).
Our work has also benefited from the AI Cluster ANITI and the research program DEEL.\footnote{\url{https://www.deel.ai/}} ANITI is funded by the France 2030 program under the Grant agreement n°ANR-23-IACL-0002. DEEL is an integrative program of the AI Cluster ANITI, designed and operated jointly with IRT Saint Exupéry, with the financial support from its industrial and academic partners and the France 2030 program under the Grant agreement n°ANR-10-AIRT-01.

\bibliography{main}
\bibliographystyle{unsrtnat}

\newpage
\appendix

\setcounter{figure}{0}
\setcounter{table}{0}
\renewcommand{\thefigure}{A\arabic{figure}}
\renewcommand{\thetable}{A\arabic{table}}
\renewcommand{\theHfigure}{A\arabic{figure}}
\renewcommand{\theHtable}{A\arabic{table}}

\section{Extended Related Work}
\label{app:extended_RW}

\subsection{Unifying existing methods}

\paragraph{} \Cref{fig:framework} summarizes the different methods for concept extraction, with which combination of properties they are designed to satisfy, as well as whether they require instance-wise correspondences for training and inference. We discuss these methods in more detail below. Our framework can be used to regularize for any combination of properties.

\paragraph{SAE.} The vanilla SAE framework only considers isolated representation instances and therefore only enforces the demi-cycle consistency property $\Pdcy$ through the reconstruction loss $\Ldcy$. This is the most basic concept extraction method, and does not enforce any particular alignment between the concepts extracted from different representation instances.

\paragraph{Transcoders.} \citep{dunefsky2024transcoders} only enforce the translation property $\Ptr$ through the translation loss $\Ltr$. This method was designed to learn a sparse alternative to a dense MLP, not to compare the conceptual alignment between two representations.

\paragraph{Crosscoders.} As discussed above, crosscoders \citep{lindsey2024sparse} were designed for concept alignment. As a result, they enforce $\Pcont$ architecturally and regularize for $\Pdcy$ (and equivalently $\Ptr$). Given their strong architectural constraint for a perfect satisfaction of $\Pcont$, they trivially satisfy the alignment duality of \Cref{prop:PtrPcont}. Indeed, they introduce $\psi_i' = \sum_{j \in I} \psi_j$. \citet{minder2025overcoming} finds that in selective encoders such as top-k or batchtopk, enforcing selection on $\psi_i'$ creates a strong pressure for the $\psi_i$ to satisfy $\Pcont$ in the form of an ``$L_0$ budget''. Splitting is therefore strongly disfavored. However, it is not impossible, for example, in the case where feature activation energies span several orders of magnitude and noise on high energy features outweighs the signal on low energy ones. In such a scenario, high-energy features benefit from splitting in order to perfectly reconstruct their instance-specific noises at the expense of some $L_0$ budget that is therefore not spent on low-energy features. On top of these hypothetical failure modes, we show in \cref{sec:recipe} that the architectural trick of introducing $\psi_i'$ renders $\psi_i$ incapable of acting as an isolated encoder, failing to satisfy all properties of reconstruction $\Pdcy$, translation $\Ptr$, and concept consistency $\Pcont$.

\paragraph{USAE.} The Universal SAE (USAE) \citep{thasarathan2025universal} was introduced as an alternative to crosscoders in the context of vision models, with the same motivation of learning a shared dictionary. In their case, they studied cross-model alignment. Their method discarded the architectural trick of crosscoders and replaced it with $\Ltr$ to regularize for translation. However, as shown by \citet{devillers2024semi}, this is not sufficient to guarantee the concept consistency property $\Pcont$ unless the decoders $\psi^{-1}_i$ are injective. In the case of SAEs, decoders are linear maps from a $K$-dimensional concept space to a $d$-dimensional latent space, with $K > d$, and are therefore not injective on the full space. However, what matters is injectivity on the support of the shared concept space $\bigcup_{i \in I} \operatorname{supp}(\psi_i \# \mu_i)$. In practice, due to the sparsity constraints of the encoders coupled with translation constraints, failing to align codes would require the decoders to have split features and each encoder to use an arbitrary subset of these split features. A sufficient condition to fall in such a case is for the SAE to be locked in a local minimum.
As an illustrative example, take the case where a shared high-energy feature $f$ was independently discovered by both SAEs at indices $k_i$ and $k_j$ respectively, as part of the reconstruction objective. Then, with high probability, $\psi_i$ at index $k_j$ corresponds to a low-energy feature. In such a case, the translation loss dominates the reconstruction one for this feature, such that $\psi_i$ learns to not encode meaningful information there, and the decoder $\psi_i^{-1}$ learns to copy $f$ at index $k_j$. This locks the SAEs in a local minimum where, in order to preserve the constraints of sparsity, reconstruction, and translation, this feature needs to stay split.
Note that there is no equivalent of the $L_0$ budget pressure of crosscoders that would prevent such a failure mode, since here, the encoders' sparsity is enforced separately for all $i \in I$. Adding a regularization term for $\Pcont$ is therefore necessary to guarantee alignment in concept space.

\paragraph{SAE-A.} The Aligned SAE (SAE-A) \citep{dhimoila2026cross} was introduced in the context of multimodal encoders explicitly tasked to learn a shared representation across modalities. As a result, they share the weight of image and text SAEs: $\psi_{\mathrm{img}} = \psi_{\mathrm{txt}}$ and enforce both $\Pdcy$ for regular concept extraction and $\Pcont$ for alignment. They formulate their DGP and assumptions under the distributional case of $\Pcontdist$ for generality. However, due to the existence of matching pairs between the two modalities in CLIP-like encoders \citep{radford2021}, they operationalize it as a regularization term on the codes of matching pairs. They further show that a very small weight on this regularization term is sufficient to get the necessary inductive bias and align codes.

\subsection{The CoSAE}

Given the unifying framework described above, it is now immediately possible to design new methods with arbitrary combinations of regularizations based on which properties are relevant under a specific experimental setting. One can even consider a mixture of instance-wise and distributional regularization terms, e.g., based on the availability and quality of matching pairs of data points. This framework also allows for possible curriculum learning schemes for maximal sample efficiency.

\section{\InterVench}
\label{app:InterVench}

\paragraph{} The \InterVench is a battery of tests designed to evaluate the quality of concepts extracted by SAEs. It is based on SAEBench~\cite{karvonen2025saebench}, a similar benchmark for evaluating disentanglement and interpretability of learned representations in the context of large language models. We adapt the relevant metrics to our setting and use them to evaluate the quality of the concepts extracted by our CoSAE under different regularization regimes. Specifically, we consider the following metrics, which target \emph{latent geometry} and \emph{concept space} quality, respectively. All of our metrics are intervention-based.

\paragraph{\boldmath$R^2$ reconstruction.} We measure the $R^2$ of the reconstruction of the latent representations from the codes, estimated as $R^2 = 1 - \frac{\|\xvec - \widehat{\xvec}\|^2}{\|\xvec - \bar{\xvec}\|^2}$, where $\xvec$ is the original latent representation, $\widehat{\xvec}$ is the reconstructed representation from the codes, and $\bar{\xvec}$ is the mean of the original representations. This metric's evaluation is in \emph{latent space}.

\paragraph{Faithfulness.} Quantifies how much of the information encoded in the latent representations is captured by the codes. Measured as the drop in model's loss when replacing the original latent representations with the reconstructed ones from the codes: $\frac{\widehat{\mathcal{L}} - \mathcal{L}_0}{\mathcal{L} - \mathcal{L}_0}$, where $\mathcal{L}$ is the original model's loss, $\widehat{\mathcal{L}}$ is the loss of the model composed with the SAE, and $\mathcal{L}_0$ is the same as $\widehat{\mathcal{L}}$ but the codes are patched with zeros. This metric's evaluation is in \emph{latent space}.

In the vision setting, we consider the training loss for SigLIP, and the loss of a linear probe on ImageNet for ViT (its original loss), as well as for DinoV2 (as we do not have access to its original loss).

\paragraph{Sparse Probing.} Measures the extent to which codes capture downstream task-relevant information in a sparse manner. We split this metric into two, one acting in \emph{latent space} and the other in \emph{concept space}. For the former, we parameterize a standard probe $\xvec^T \thetavec$ as a sparse linear combination of dictionary atoms $\thetavec \coloneqq \sum_{i=0}^{K-1} \alpha_i \dvec_i \in \R^d$---$\mathrm{supp}((\alpha_i)) \ll d$---instead of as a dense combination of canonical basis vectors $\thetavec \coloneqq \sum_{i=0}^{d-1} \alpha_i \evec_i$.
For the latter, the probe becomes $\zvec^T \thetavec$ instead of $\xvec^T \thetavec$, where $\zvec = E(\xvec)$ is the code of $\xvec$. $\thetavec$ is now equal to $(\alpha_i)$, still with $\mathrm{supp}((\alpha_i)) \ll d$. In SAEBench, only the latter is considered. The score we report is the raw accuracy of the probe. In the vision setting, we report the sparse probing score on ImageNet.

All subsequent metrics are probe-based and can therefore be adapted in the same way to both latent space and concept space, with the same distinction as above. They all go beyond sparse probing by measuring not only the amount of downstream task-relevant information captured by dictionary atoms, but also how disentangled this information is across atoms. Due to the necessity for the following metrics to indicate that sparse probing is successful, we force them to zero when the sparse probing score is below a certain threshold.

\paragraph{Unlearning.} Measures whether the SAE grants the ability to selectively unlearn specific information (on a \emph{forget task}, specific) without affecting other information (measured on a \emph{utility task}, general). We measure both (\emph{i}) $u = \frac{\mathrm{acc}^{\mathrm{utility}}_{\mathrm{unlearned}} - \mathrm{acc}^{\mathrm{utility}}_0}{\mathrm{acc}^{\mathrm{utility}}_{\mathrm{original}} - \mathrm{acc}^{\mathrm{utility}}_0}$ the utility score, where $\mathrm{acc}_0$ denotes the accuracy of a probe trained on representations where codes $z$s are patched with zeros, $\mathrm{acc}_{\mathrm{original}}$ the accuracy of a probe trained on original representations, and $\mathrm{acc}_{\mathrm{unlearned}}$ the accuracy of a probe trained on representations where codes labeled as relevant for the forget task are patched with zeros. We also measure (\emph{ii}) $f = 1 - \frac{\mathrm{acc}^{\mathrm{forget}}_{\mathrm{unlearned}} - \mathrm{acc}^{\mathrm{forget}}_0}{\mathrm{acc}^{\mathrm{forget}}_{\mathrm{original}} - \mathrm{acc}^{\mathrm{forget}}_0}$ the forget score. The final unlearning score is then computed as $u \cdot f$.

Other combinations of utility and forget scores are possible, such as $u + f$, where losing one point of utility is worth gaining one point of forget. We do not consider this combination, as it doesn't penalize mediocre tradeoffs where both utility and forget scores are around $0.5$. More generally, the score can be the distance to the Pareto optimal $(1, 1)$: $-\|(u, f) - (1, 1)\|_p$ for some $p \in [1, +\infty]$, with easy geometric interpretations.

In the vision setting, we take ImageNet multiclass classification as the utility task, and CelebA \citep{liu2015deep} multilabel classification as the forget task.

\paragraph{Spurious Correlation Removal (SCR).} Measures whether the SAE disentangles spurious correlations in the data. We report $\mathrm{SCR} = \frac{\mathrm{acc}_{\mathrm{decorr}} - \mathrm{acc}_{\mathrm{original}}}{\mathrm{acc}_{\mathrm{oracle}} - \mathrm{acc}_{\mathrm{original}}}$, where $\mathrm{acc}_{\mathrm{original}}$ is the accuracy of a probe trained on the dataset where the spurious correlation is present (dataset A), $\mathrm{acc}_{\mathrm{oracle}}$ the accuracy of a probe trained on a fully de-correlated dataset (dataset B). For $\mathrm{acc}_{\mathrm{decorr}}$, we take the sparse probe trained on dataset A and evaluate the contribution of each atom to the spurious correlation. We zero out the atoms whose contribution is above a certain threshold, and retrain a probe on the resulting support.

In the vision setting, we select the waterbirds dataset \citep{Sagawa2020Distributionally}, where the spurious correlation is that between the background (water vs land) and the class (waterbird vs land-bird). Dataset A is the training set, while dataset B is a balanced test set where the correlation is removed. For some reason, the models we used, i.e., DinoV2, ViT, and SigLIP, do not seem to catch the spurious correlation in the first place, rendering this metric noisy and uninformative. For this reason, we do not use it in our final evaluation. Future work should explore other known or synthetic spurious correlations in vision, such as color MNIST \citep{arjovsky2019invariant}.

\paragraph{Targeted Probe Perturbation (TPP).} Measures whether the SAE disentangles the features relevant to distinct classes across non-overlapping sets of atoms, in a multiclass classification task. Similar to unlearning, we report $u \cdot f$. In this case, we select $c$ a class of interest, or a set of classes, and $\overline{c}$ its complement. We then consider $c$ as the forget task, and $\overline{c}$ as the utility task.

In the vision setting, we take ImageNet multiclass classification and report the TPP score averaged across subsets $c$ of $10$ classes, selected at random.

\paragraph{\boldmath{$\Mextract$}.} As stated before, we remove the SCR metric from our final evaluation in the vision setting. We additionally zero out the unlearning and TPP scores when the corresponding sparse probing score is below a certain threshold, typically around $0.1$ in the ImageNet multiclass classification setting. We further split the metrics as follows. \InterVenchLatent contains metrics that measure the ability of the dictionary atoms to characterize the latent space. These metrics are $R^2$ reconstruction and faithfulness, as well as the \emph{latent space} version of sparse probing, unlearning, and TPP. \InterVenchConcept contains metrics that evaluate the quality of the concept space as its own representation space. These metrics are the \emph{concept space} version of sparse probing, unlearning, and TPP. $\Mextract$ is then computed as the average of all metrics in \InterVench.

\paragraph{Cycle Consistency.} Separately from the aggregate of the above score, we also report the $R^2$ associated with the cycle consistency. The only difference with the $R^2$ of the reconstruction is that $\widehat{\xvec}$ is replaced by $\widehat{\xvec} \coloneqq \psi_i^{-1} \circ \psi_j \circ \psi_j^{-1} \circ \psi_i (\xvec)$.

\section{\InterVenchA}
\label{app:InterVenchA}

\paragraph{} The \InterVenchA benchmark is a natural extension of the \InterVench to the evaluation of concept alignment quality. It is based on the same metrics as the \InterVench, but adapted to evaluate the quality of the alignment between two SAEs trained on different domains. In order to measure this, we do not rely on the ability of SAEs to capture capabilities, but to \emph{transfer} them. Specifically, the adaptation of \InterVenchLatent metrics measures the quality of \emph{translation} properties, while the adaptation of \InterVenchConcept metrics measures the quality of \emph{concept consistency} properties.

\paragraph{Reconstruction.} The natural extension of $R^2$ reconstruction and faithfulness metrics for alignment is to consider the following. Let $i$ and $j$ index two domains. Let $\xvec_i$ and $\xvec_j$ be joint samples from domains $i$ and $j$ respectively. Let $E_i$ and $D_i$ be the encoder and decoder of the SAE trained on domain $i$, similarly for $E_j$ and $D_j$. We now consider the reconstruction $\widehat{\xvec}_j = D_j(E_i(\xvec_i))$ instead of $\widehat{\xvec}_j = D_j(E_j(\xvec_j))$.

\paragraph{Sparse Probing.} Given domains $i$ and $j$, we train a sparse probe exactly as before on domain $i$. This results in a sparse vector of parameters $\alphavec \in \R^K$, where $K$ is the number of dictionary atoms and $\mathrm{supp}(\alphavec) \ll d$. In the \emph{latent} setting, this corresponds to $\thetavec_i \coloneqq \alphavec^T \Dmat_i$, where $D_i$ is the dictionary of domain $i$. We therefore evaluate the transfer capabilities of jointly trained SAEs by measuring the resulting accuracy of the probe $\thetavec_j \coloneqq \alphavec^T \Dmat_j$ on domain $j$. This measures alignment as the ability of dictionary atoms to jointly characterize each of their domain's geometry. In the \emph{concept} setting, we have $\thetavec_i \coloneqq \alphavec$, evaluated on domain $i$ through the accuracy of the classifier $E_i(\xvec_i)^T \thetavec_i$. We then evaluate the accuracy of $\thetavec_j \coloneqq \thetavec_i$ on domain $j$. This measures \emph{concept consistency} as the ability of the same set of concepts to capture the same task-relevant information across domains.

\paragraph{Unlearning, SCR and TPP}'s extensions naturally follow from the above. We design interventions on domain $i$ and transfer them to domain $j$ to evaluate the quality of the alignment.

\paragraph{\boldmath{$\Malign$}.} All post-processing of the metrics is the same as for $\Mextract$. All metrics now give a matrix of scores. The diagonal, i.e., when $i = j$, corresponds to the original \InterVench metrics on each domain separately. The off-diagonal elements correspond to the transfer capabilities of the SAEs across domains. $\Mextract$ is the average across metrics of the average of the diagonal, while $\Malign$ is averaged across the off-diagonal elements. $\Mtr$ specifically targets translation properties, averaging the $R^2$ and faithfulness metrics as well as the \emph{latent space} version of sparse probing, unlearning and TPP. $\Mcont$ on the other hand specifically targets concept consistency properties, averaging the \emph{concept space} version of sparse probing, unlearning, and TPP.

\paragraph{Note.} In the special setting of multimodal encoders, we select the following metrics. The $R^2$ is unchanged. For the faithfulness, we need to adapt it to the case of two-tower multimodal encoders. We report $\frac{\widehat{\mathcal{L}}_{ij} - \mathcal{L}_0}{\mathcal{L} - \mathcal{L}_0}$, where $\mathcal{L}$ is the original contrastive loss of the two towers, $\widehat{\mathcal{L}}_{ij}$ is the same loss but where the original representations of domain $i$ are translated as $\widehat{x}_j = D_j(E_i(\xvec_i))$, and $\mathcal{L}_0$ is the same as $\widehat{\mathcal{L}}_{ij}$ but where the codes are patched with zeros. This gives a $2\times 2$ matrix of scores. $\Mextract$ and $\Mtr$ are computed using these scores as described above. To get $\Mcont$, we derive the faithfulness as follows. We replace $\widehat{\mathcal{L}}_{ij}$ with $\widehat{\mathcal{L}}$, the contrastive loss of the model composed with the encoders only. This corresponds to using the concept space as a surrogate for the latent space for contrastive learning. This gives a single score, which we use as the faithfulness metric for $\Mcont$. Additionally, we include ImageNet zero-shot classification accuracy as a downstream task for all $\Mextract$, $\Mtr$ and $\Mcont$ respective cases.

\section{Human Validation}
\label{app:human_validation}

We ran a human evaluation study for additional verification that alignment score agrees with human judgment. Most of our metrics rely on semantic categories drawn from human annotations (via ImageNet labels used in sparse probing), but do not directly query human judgment of alignment itself. This experiment checks whether alignment scores derived that way agree with direct human perceptual judgment of cross-domain feature correspondence. The protocol was approved by the IRB of an author-affiliated institution. We proceed as follows.

\paragraph{Participant selection.} We recruited 8 voluntary participants among colleagues and relatives. Participants were not aware of the goal of the study beyond what the welcome screen described. The study was conducted online. Participants provided informed consent and did not receive monetary compensation for their time. They were instructed to play the game for as long as they desired, which averaged a duration of $\approx23$min per participant. We collected a total of 2063 individual trials --- $\approx688$ per experimental cell.

\paragraph{Welcome Screen.} The welcome screen displays both a short description of the task they were about to do, as well as a consent form.

\paragraph{Practice Session.} Before the main experiment began, participants went through 10 practice runs with hand crafted features, whose purpose was to explain the task, preceded by an instruction screen. \cref{fig:instructionscreen} shows example trials.

\begin{figure}[h]
    \centering
    \includegraphics[width=0.5\linewidth]{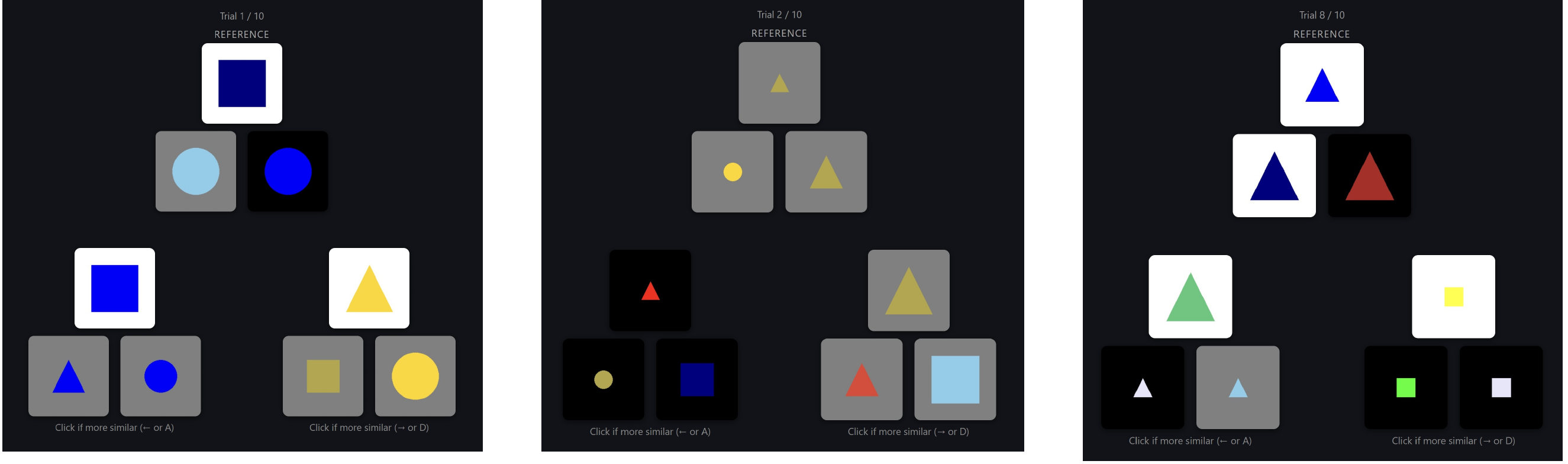}
    \caption{Example practice trials of the human evaluation study.}
    \label{fig:instructionscreen}
\end{figure}

\paragraph{Main Experiment.} The main experiment follows the same format as the practice sessions, but with real data. An example is shown in \cref{fig:trialscreen}. Participants are presented with a “REF” (reference) set of maximally activating images among ImageNet for one dictionary atom from one domain (e.g. dino). Say this is atom $k$ of dictionary $D_i$. Then, they are shown two sets “A” and “B” of mutually similar images. For set “A”, we select atom $k$ of dictionary $D_j$. For set ``B'', we select an atom $k' \neq k$ from $D_j$ at random. Choices “A” or “B” are presented next to one another, with an order selected at random. Participants are asked to click on the set they find most similar to ``REF''.

\begin{figure}[h]
    \centering
    \includegraphics[width=0.5\linewidth]{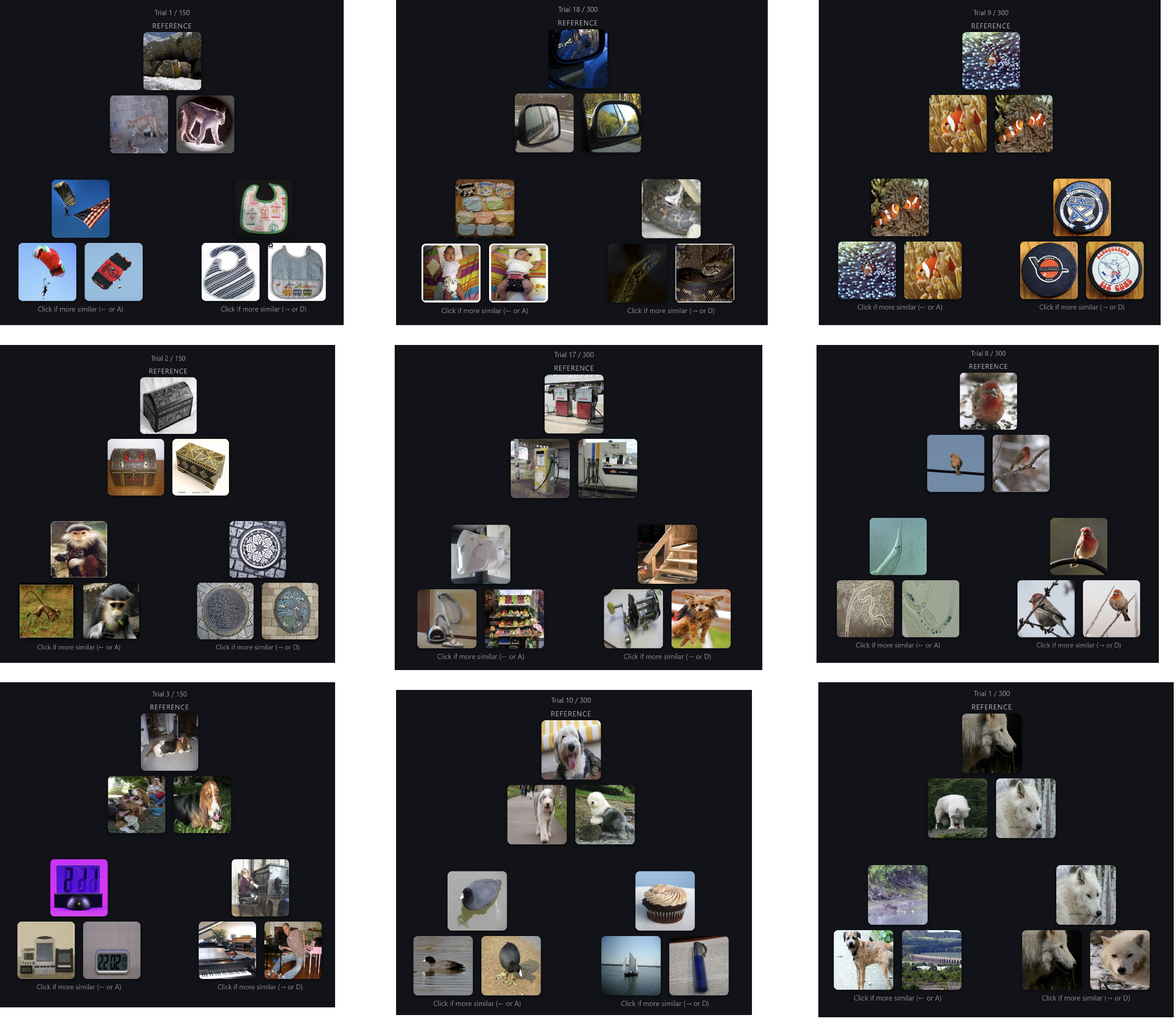}
    \caption{Example trials from the ``control'' (\emph{left}), ``aligned bad'' (\emph{center}) and ``aligned good'' (\emph{right}) cells.}
    \label{fig:trialscreen}
\end{figure}

Each trial was one of three pre-defined cells: (\emph{i}) a control, where SAEs were trained without any specific alignment regularisation across domains, and atoms are selected at random among the ones whose corresponding sparse probe is successful (accuracy $>0.9$); (\emph{ii}) an “aligned bad” condition, where features are selected from CoSAEs such that InterVenchA score is low --- specifically, such that sparse probing is successful (accuracy of the sparse probe $> 0.9$) but the transfer accuracy is low ($< 0.2$). This means that this feature in particular is good quality but alignment failed (according to our metrics). (\emph{iii}) An “aligned good” condition, where both the sparse probing is successful (same condition) and the transfer accuracy remains high ($>0.9$). This means that the feature is good and the alignment is good.

The goal is to measure the probability that humans chose the set “A”, and whether it is significantly higher in “aligned good” than in “aligned bad” trials. A statistically significant difference means that our measure of alignment quality tends to agree with human perception of successful alignment.

\paragraph{Results.}

\begin{table}[h]
\centering
\begin{tabular}{lcc}
\toprule
Cell & $P(\text{choose A})$ & Cluster-bootstrap 95\% CI \\
\midrule
Control       & $0.517$ & $[0.500, 0.540]$ \\
Aligned bad   & $0.586$ & $[0.569, 0.618]$ \\
Aligned good  & $0.665$ & $[0.625, 0.696]$ \\
\bottomrule
\end{tabular}
\caption{Probability of choosing the correct set (``A'') across experimental cells.}
\label{tab:choose_a}
\end{table}

The absolute differences are $+6.9$ percentage points for aligned-bad versus control, $+14.8$ points for aligned-good versus control, and $+7.9$ points for aligned-good versus aligned-bad.

A more detailed crossed random-effects binomial mixed model, accounting for participant and feature variability, gives the following contrasts. Because the fitted output reports posterior means and posterior standard deviations, we report posterior odds ratios and normal-approximation $95\%$ credible intervals rather than frequentist $p$-values. We fit a model using \texttt{statsmodels.genmod.bayes\_mixed\_glm.BinomialBayesMixedGLM}.

\begin{table}[htbp]
\centering
\resizebox{\linewidth}{!}{%
\begin{tabular}{lccc}
\toprule
Contrast & Posterior log-odds $\pm$ SD & Odds Ratio (OR) & Approx.\ 95\% CI for OR \\
\midrule
Aligned bad vs.\ control        & $0.621 \pm 0.091$ & $1.85$ & $[1.56, 2.22]$ \\
Aligned good vs.\ control       & $1.013 \pm 0.099$ & $2.75$ & $[2.27, 3.34]$ \\
Aligned good vs.\ aligned bad   & $0.405 \pm 0.097$ & $1.50$ & $[1.24, 1.81]$ \\
\bottomrule
\end{tabular}
}
\caption{Pairwise contrasts between cells from the crossed random-effects binomial mixed model (participant and feature as random effects).}
\label{tab:contrasts}
\end{table}

The feature-level random-effect standard deviation is $1.566; [1.446, 1.701]$, whereas the participant-level standard deviation is $0.195; [0.091, 0.444]$, indicating substantial heterogeneity across features and supporting the use of crossed random effects.

\paragraph{Conclusion.} The human result provides graded criterion validity for the probe ranking: aligned probes are significantly preferred to control probes, and the probes labeled ``aligned good'' are significantly preferred to those labeled ``aligned bad'', even after accounting for participant and feature variability.

This result does not prove that every selected feature is semantically correct, validate every InterVenchA component, or remove other confounds common in SAE interpretability. It shows that the good-versus-bad probe distinction provided by the benchmark is reflected in human judgments.

\section{Linear case: analytical solutions}
\label{app:linear_analysis}

\paragraph{}In this section, we study the analytical solutions of the proposed framework in the linear setting. By restricting all encoders and decoders to be linear maps, we can explicitly characterize the effect of each objective and identify which components of the model (encoders, decoders, or their compositions) are recovered. This analysis highlights how different combinations of losses induce different notions of coupling across and within domains.

\paragraph{Notation.}
Let $\Xmat \in \R^{n \times d_x}$ and $\Ymat \in \R^{n \times d_y}$ be centered data matrices corresponding to two representation instances. We denote their empirical covariance matrices by
\begin{equation}
\Sigmamat_{xx} = \tfrac{1}{n} \Xmat^\top \Xmat, \quad
\Sigmamat_{yy} = \tfrac{1}{n} \Ymat^\top \Ymat, \quad
\Sigmamat_{xy} = \tfrac{1}{n} \Xmat^\top \Ymat.
\end{equation}
We consider linear encoders and decoders of the form
\begin{equation}
\Zmat_x = \Xmat \Wmat_x, \quad \Zmat_y = \Ymat \Wmat_y, \qquad
\widehat{\Xmat} = \Zmat_x \Vmat_x, \quad \widehat{\Ymat} = \Zmat_y \Vmat_y,
\end{equation}
where $\Wmat \in \R^{d_x \times K}$ are encoder matrices, and the corresponding decoders are given by $\Vmat \in \R^{K \times d_x}$. Translation between the two domains is defined as
\begin{equation}
\widehat{\Ymat} = \Xmat \Wmat_x \Vmat_x, \qquad \widehat{\Xmat} = \Ymat \Wmat_y \Vmat_y.
\end{equation}
When needed, we impose whitening constraints of the form
\begin{equation}
\frac{1}{n} \Zmat_x^\top \Zmat_x = \Imat_K, \qquad \frac{1}{n} \Zmat_y^\top \Zmat_y = \Imat_K.
\end{equation}
All expectations in the losses are understood empirically over the dataset. Objectives are given by the losses defined in \cref{sec:operationalizing}. We analyze these objectives and their combinations in the linear setting, and characterize the corresponding solutions. In particular, we identify which components (encoders, decoders, or their compositions) are uniquely determined, and how different losses induce coupling across and within domains.

\subsection{Reconstruction}
\label{app:linear_dcy}

We begin with the case where no coupling is introduced between $\Xmat$ and $\Ymat$, i.e., when only self-reconstruction objectives are considered. In this setting, the two domains decouple, and each reduces to a standard approximation problem.

\begin{lemma}[PCA]
\label{lem:pca}
Consider the self-reconstruction objectives
\begin{equation}
\Ldcy^x = \frac{1}{n} \|\Xmat - \Xmat \Wmat_x \Vmat_x\|_F^2,
\qquad
\Ldcy^y = \frac{1}{n} \|\Ymat - \Ymat \Wmat_y \Vmat_y\|_F^2.
\end{equation}
Then, for each domain independently, any minimizer satisfies:
\begin{itemize}
    \item The encoder $\Wmat_x$ (resp.\ $\Wmat_y$) spans the top-$K$ principal subspace of $\Sigmamat_{xx}$ (resp.\ $\Sigmamat_{yy}$).
    \item The optimal decoder is given by $\Vmat_x = \Wmat_x^\top$ (resp.\ $\Vmat_y = \Wmat_y^\top$) up to a change of basis.
\end{itemize}
\end{lemma}

\begin{proof}
    See \cref{app:proof_pca}.
\end{proof}

\paragraph{Discussion.}
Self-reconstruction identifies both encoders and decoders within each modality, but introduces no coupling between $\Xmat$ and $\Ymat$. In particular, the solutions are invariant under invertible transformations of the latent space, i.e., for any invertible matrix $\Rmat \in \R^{K \times K}$, the transformation
\begin{equation}
\Wmat_x \rightarrow \Wmat_x \Rmat,
\qquad
\Vmat_x \rightarrow \Rmat^{-1} \Vmat_x
\end{equation}
leaves the objective unchanged. As a result, the latent representations are only identifiable up to a change of basis. Imposing whitening constraints selects an orthonormal basis of this subspace, corresponding to the principal directions (i.e., the right singular vectors of $\Xmat$).

\subsection{Linear Coupling}
\label{app:linear_coupling}

We now turn to objectives that introduce coupling between the two domains. Unlike self-reconstruction, which operates independently on each modality, these objectives enforce interactions either at the level of representations (concept consistency) or through predictive mappings (translation). We study how these different forms of coupling affect identifiability and the structure of the learned representations.

\subsubsection{Concept consistency}
\label{app:linear_cont}

We first consider coupling at the level of latent representations, through the concept consistency objective.

\paragraph{Objective.}
The concept consistency loss is defined as
\begin{equation}
\Lcont = \frac{1}{n} \|\Zmat_x - \Zmat_y\|_F^2
= \frac{1}{n} \|\Xmat \Wmat_x - \Ymat \Wmat_y\|_F^2.
\end{equation}

\paragraph{CCA.} We begin by showing that the concept consistency objective alone recovers the canonical correlation analysis (CCA) for the encoders.

\begin{lemma}[Degeneracy of concept consistency]
\label{lem:cont_degenerate}
In the absence of additional constraints, minimizing $\Lcont$ is ill-posed: for any pair $(\Wmat_x, \Wmat_y)$, and any orthogonal matrix $\Rmat \in \R^{K \times K}$, the transformation
\begin{equation}
\Wmat_x \rightarrow \Wmat_x \Rmat, \qquad
\Wmat_y \rightarrow \Wmat_y \Rmat
\end{equation}
leaves the objective unchanged. Moreover, the objective admits degenerate solutions (e.g., collapse to zero).
\end{lemma}

\begin{proof}
See \cref{app:proof_cca}.
\end{proof}

\begin{lemma}[CCA from concept consistency with whitening]
\label{lem:cca}
Consider minimizing $\Lcont$ under the whitening constraints
\begin{equation}
\frac{1}{n} \Zmat_x^\top \Zmat_x = \Imat_K, \qquad
\frac{1}{n} \Zmat_y^\top \Zmat_y = \Imat_K.
\end{equation}
Then the optimal encoders $(\Wmat_x, \Wmat_y)$ are given by the top-$K$ canonical directions between $\Xmat$ and $\Ymat$, i.e., they solve the canonical correlation analysis (CCA) problem.
\end{lemma}

\begin{proof}
See \cref{app:proof_cca}.
\end{proof}

\paragraph{Discussion.}
Concept consistency introduces coupling across modalities by enforcing alignment of latent representations. However, without additional constraints, the problem is ill-posed due to scaling and collapse. Whitening resolves these ambiguities and yields CCA, which identifies shared directions maximizing cross-correlation between $\Xmat$ and $\Ymat$.

Notably, this objective only identifies encoders: decoders remain unconstrained, and no reconstruction structure is imposed within each modality.

\paragraph{Identifying the Decoders.} The CCA alone discards the decoders. To identify them, we now add self-reconstruction and consider the combined objective
\begin{equation}
\Lcont + \Ldcy^x + \Ldcy^y.
\end{equation}

\begin{lemma}[Trade-off between CCA and PCA]
\label{lem:cont_dcy}
Minimizing $\Lcont + \lambda \Ldcy$ yields encoders $(\Wmat_x, \Wmat_y)$ that solve a generalized eigenvalue problem interpolating between principal components and canonical directions. After accounting for degeneracies,
\begin{itemize}
    \item (\emph{i}) the solution does not depend on $\lambda$,
    \item (\emph{ii}) only the pure CCA remains,
    \item (\emph{iii}) the effect of the reconstruction term is to identify the decoders as optimal linear readouts
\end{itemize}
\end{lemma}

\begin{proof}
See \Cref{lem:cont_dcy_proof}.
\end{proof}

\paragraph{Discussion.}
While the combined objective highlights a fundamental trade-off between shared and idiosyncratic information, this trade-off disappears after solving for degenerate identification of the shared subspace. Only shared information remains, and the effect of self-reconstruction is merely to provide solutions to the decoders, as the raw CCA only considers encoders.

\subsubsection{Translation}
\label{app:linear_tr}

We now consider coupling through translation, which enforces predictive mappings between the two domains. Unlike concept consistency, which aligns latent representations, translation constrains the composition of encoders and decoders across modalities, and therefore couples representations through their ability to reconstruct one domain from the other.

In contrast with concept consistency, which leaves decoders entirely unconstrained and requires self-reconstruction to identify them, translation introduces a different form of degeneracy: it constrains only cross-modal compositions (e.g., $\Wmat_x \Vmat_y$), but does not couple encoders and decoders within each modality. As a result, the individual components remain underdetermined. As before, adding self-reconstruction restores this intra-modality coupling and leads to well-identified solutions.

\paragraph{RRR.}First, consider the translation objective alone. This corresponds to the reduced-rank regression (RRR) and only identifies the composition of encoders from one domain to decoders of the other domain. Encoders and decoders within one domain are decoupled, and all operators are defined up to some transforms.

\begin{lemma}[Translation and reduced-rank regression]
\label{lem:tr_rrr}
Consider the (one-directional) translation objective
\begin{equation}
\Ltr^{x \to y}(\Wmat_x, \Vmat_y)
=
\frac{1}{n} \|\Ymat - \Xmat \Wmat_x \Vmat_y\|_F^2.
\end{equation}
Then minimizing $\Ltr^{x \to y}$ over $(\Wmat_x, \Vmat_y)$ is equivalent to RRR of $\Ymat$ onto $\Xmat$ with rank constraint $K$. The factorization $(\Wmat_x, \Vmat_y)$ is not identifiable: only the product $\Wmat_x \Vmat_y$ is. One such factorization recovers the CCA.
\end{lemma}

\begin{proof}
    See \cref{lem:tr_rrr_proof}.
\end{proof}

\paragraph{Adding self-reconstruction.}
We now add self-reconstruction objectives, which introduce coupling within each modality:
\begin{equation}
\Ltr + \lambda (\Ldcy^x + \Ldcy^y).
\end{equation}
While translation alone only constrains the cross-modal compositions, self-reconstruction couples encoders and decoders within each domain. This restores intra-modality structure and yields well-defined decoders given the encoders, but does not fully resolve the non-identifiability of the encoders. Further whitening is still required, as in the concept consistency case.

\begin{lemma}[Translation with reconstruction]
\label{lem:tr_dcy}
Under this combined objective, the optimal encoders are again given by the CCA, while the decoders are given by an interpolation between the solutions of \Cref{lem:tr_rrr} and \Cref{lem:cont_dcy}.
\end{lemma}

\begin{proof}
    See \cref{lem:tr_dcy_proof}.
\end{proof}

\subsubsection{Distributional alignment}
\label{app:linear_dist}

\paragraph{} Distributional alignment enforces matching of pushforward measures rather than individual samples. Intuitively, it aligns global statistics (e.g., projections or moments) of the representations across domains. However, this does not introduce any structural coupling between $X$ and $Y$: many unrelated mappings can induce identical distributions. As a result, the correlation structure may align in aggregate, but the underlying representations remain uncoupled, and no meaningful joint factorization is recovered. This is true for both translation and concept consistency.

\subsubsection{Cycle consistency}
\label{app:linear_cy}

\paragraph{} Cycle consistency enforces that $\Xmat \simeq \Xmat \Amat \Bmat$ and $\Ymat \simeq \Ymat \Bmat \Amat$ with $\Amat = \Wmat_X \Vmat_Y$ and $\Bmat = \Wmat_Y \Vmat_X$, introducing an apparent coupling between the two domains. Without coupling, the optimal solution reduces to independent PCA decompositions, say $P_x$ and $P_y$. With cycle consistency, one recovers the same PCA solution up to an arbitrary invertible transform: $\Amat = \Pmat_x \Rmat \Pmat_y^\top$ and $\Bmat = \Pmat_y \Rmat^{-1} \Pmat_x^\top$ (\Cref{lem:cycle_proof}). The factorization remains ambiguous up to an invertible transform, and the optimal solution is still equivalent to PCA, meaning the apparent coupling vanishes upon solving the problem.

\section{Platonic and Linear Representation Hypothesis}
\label{app:prh_lrh}

\paragraph{Platonic Representation Hypothesis (PRH).} The PRH is a universal statement that any representation learning method will converge to the same shared representation, up to some transform \citep{huh2024platonic}. This can be formalized in various ways. We chose to write this hypothesis under a generative framework as it allows for a formulation that is ready to be plugged into our above framework. The generative view of the PRH assumes the existence of a platonic latent space $X$ such that all real world observations, whatever their domain $\mathcal{X}^i$, stem from this platonic latent space through a generator $g_i : X \to \mathcal{X}^i$. Then, all feature extractors $\phi_i : \mathcal{X}^i \to X^i$ invert this generative process up to a transform $\psi_i$. In the context laid out above, adding this extra generative step allows for intermediate observation spaces $(\mathcal{X}^i)_{i \in I}$ between the platonic latent space $X$ and the feature extractors $\phi_i$. This is especially relevant when the $g_i$s are allowed to introduce distribution shifts or even modality shifts. For simplicity, we can further add the restriction that $X = \R^K$.

\begin{hypothesis}[Platonic Generative Process]
\label{hyp:generative_process}
$\exists K\!\in\!\N, \exists X\!\subseteq\!\R^{\K}$ s.t. $\forall i\!\in\!I$
\begin{align}
    &\exists g_i:X\!\to\!\mathcal{X}^i\\
    &\forall \phi_i:\mathcal{X}^i\!\to\!X^i, \exists \psi_i:X^i\!\to\!X,\text{ s.t. }\psi_i\circ\phi_i\circ g_i=\operatorname{id}_{X}
\end{align}
\end{hypothesis}

\paragraph{Linear Representation Hypothesis (LRH).} In practice, most of the recent representation alignment and concept extraction methods implicitly assume the existence of such a generative process. In other words, they assume that representation learning inverts the generative process. In the case of representation alignment, the class of $\psi_i$ under consideration is further constrained, allowing invariance up to translation, scaling, orthogonal transformations, etc. Dictionary learning-based methods fall under a different assumption, more realistic, and grounded in neurosciences: that platonic representations are \emph{sparse} (\Cref{hyp:sparse_inversion}), imposing different constraints on $\psi_i$. This sparse coding framework often comes with additional linear or close to linear constraints. The reason is (\emph{i}) that it allows for GPU-friendly implementations through the extremely popular Sparse Autoencoder (SAE) framework, (\emph{ii}) that representation learning is explicitly designed to produce latent spaces with meaningful Euclidean---or at least linear---geometry, with downstream operations reading linearly from $X^i$ before optionally applying some nonlinearities.

\begin{hypothesis}[Sparse Generative Process]
\label{hyp:sparse_inversion}
$\xvec \in X \sim \mu=\prod_{k=1}^{\K}\marginal$ with $|\mathrm{supp}(\xvec)|\!\ll\!\K$.
\end{hypothesis}

Essentially, this generative formulation of the PRH coupled with the SAE operationalization of the LRH grounds our concept alignment framework. It does so by assuming the existence of $X$ and of the $\phi_i$, now decomposed into an additional generative step $g_i$ and an inversion step $\phi_i$, and finally, by assuming that the class of autoencoders under consideration provides the final inversion step, aligning all representations back in a shared concept space through the $\psi_i$. Our framework additionally provides for a natural distributional version of the PRH. We call the element-wise version the \emph{strong} PRH, and the distributional version the \emph{weak} PRH.

\section{Distributional losses}
\label{app:L_dist}

\paragraph{}
When instance-wise correspondences are unavailable, enforcing $\Ptrdist$ or $\Pcontdist$ requires matching distributions rather than individual samples. A natural approach is to minimize a discrepancy between pushforward measures, for instance, using optimal transport. However, distributional distances like the Wasserstein distance are both computationally expensive and poorly behaved in high dimensions and mini-batch settings, making them ill-suited for large-scale SAE training.

A practical alternative is provided by the Cramer-Wold theorem, which states that two distributions are equal if and only if their one-dimensional projections along all directions are equal. This motivates the use of sliced distributional objectives, which compare 1D projections of the distributions along randomly sampled directions. We take inspiration from \citet{balestriero2025lejepa}, who use the Epps-Pulley test on random 1D slices to regularize for Gaussianity of their embedding space:
\begin{align}
\label{eq:sliced_EP}
\mathrm{EP}(\mu_i) = b \int_{-\infty}^{\infty} \left|\widehat{\varphi}_{\mu_i}(t) - \varphi(t)\right|^2 w(t)\,\mathrm{d}t
\end{align}
where $\widehat{\varphi}_{\mu_i} = \frac{1}{b}\E_{\xvec \sim \mu_i} e^{it \langle \xvec, \uvec \rangle} = \sum_{k=1}^b e^{it \langle \xvec_{i, k}, \uvec \rangle}$ is the empirical characteristic function (ECF) of $\mu_i$ where a batch of $b$ samples are projected onto a random direction $\uvec \in \mathcal{S}^{d-1}$; $\varphi = e^{-t^2/2}$ is the characteristic function of the normal distribution, and $w$ is a weighting function---typically a gaussian. In the case of $\Ptrdist$, the target is not the normal distribution but another of our $\mu_j$, and the source becomes $\psi^{-1} \circ \psi_i \# \mu_i$. We can therefore define the sliced translation loss as follows:
\begin{align}
\label{eq:sliced_translation_loss}
\Ltrdist
&\coloneqq \E_{\uvec \sim \mathcal{S}^{d-1}} b \int_{-\infty}^{\infty} \left|\widehat{\varphi}_{(\psi^{-1}_j \circ \psi_i) \# \mu_i}(t) - \widehat{\varphi}_{\mu_j}(t)\right|^2 w(t)\,\mathrm{d}t\\
&\simeq \frac{1}{|\mathcal{U}|} \sum_{\uvec \in \mathcal{U}} b \int_{-\infty}^{\infty} \left|\widehat{\varphi}_{(\psi^{-1}_j \circ \psi_i) \# \mu_i}(t) - \widehat{\varphi}_{\mu_j}(t)\right|^2 w(t)\,\mathrm{d}t
\end{align}
where $\mathcal{U} \in (\mathcal{S}^{d-1})^m$ is a finite set of random directions, resampled at each training step. While any single slice provides only a weak constraint, aggregating over many random projections yields a consistent approximation of the target distributional alignment. Considering that $\mathcal{U}$ is resampled at each training step, the effective number of slices seen over training becomes even larger. Crucially, \citet{balestriero2025lejepa} show that a small coverage of $\mathcal{S}^{d-1}$ per step is sufficient to achieve good performance, with theoretical error bounds when the target distribution is gaussian. These theoretical bounds do not apply in our case since we do not control the target distribution. However, this approach is equivalent to regularizing with a sliced Maximum Mean Discrepancy (MMD) (\Cref{prop:EP_is_MMD}), which is widely used in practice for distributional alignment \citep{gretton2012kernel,long2015learning}, especially as an alternative to adversarial training \citep{li2015generative,binkowski2018demystifying}.

In the concept space $\mathbb{R}^K$, this perspective admits a particularly simple interpretation. In arbitrary latent spaces, canonical directions are meaningless, only the overall Euclidean geometry matters, which motivates arbitrary random slices for full coverage of the space. In contrast, in the concept space, canonical directions correspond to dictionary atoms, which are assumed to be meaningful, atomic units of information. Therefore, in concept space, we do not consider arbitrary random slices $u \in \mathcal{S}^{K-1}$, but rather canonical slices along the coordinate axes. This yields the following loss for distributional alignment of concepts:
\begin{align}
\label{eq:sliced_contrastive_loss}
\Lcontdist
&\coloneqq \frac{1}{|\mathcal{U}^*|} \sum_{\uvec \in \mathcal{U}^*} b \int_{-\infty}^{\infty} \left|\widehat{\varphi}_{\psi_i \# \mu_i}(t) - \widehat{\varphi}_{\psi_j \# \mu_j}(t)\right|^2 w(t)\,\mathrm{d}t
\end{align}
where $\mathcal{U}^* = \{\evec_1, \dots, \evec_K\}$ is the set of canonical basis vectors in $\R^K$. In practice, due to the sparsity of activations in concept space, we instead take $\mathcal{U}^*$ to be the set of active atoms across the batch. This formulation of $\Lcontdist$ is further motivated by \citet{dhimoila2026cross}'s IsoEnergy assumption.%

\section{Additional Experimental Details}
\label{app:experiment_detail}

\subsection{Synthetic Data Generating Process}
\label{app:synthetic_DGP}

\paragraph{} We design a synthetic data-generating process (DGP) that follows precisely our model. First, we sample $N$ i.i.d. samples of normally distributed latent variables $\xvec' \in \R^K$. Then, for each sample, we select its top $k$ latent variables, set them to $1$ and the rest to $0$, to get the points $\xvec \in X$. Then, for each domain $i$, we construct $\phi_i$ as the composition of the following. First, we apply some dropout to $\xvec$. Then, we add a random white noise on the active latent variables, with standard deviation $\sigma = 0.1$, and final values clipped between $1\pm\sigma$. Finally, we apply a succession of random linear transformations interleaved with ReLU nonlinearities, with a target dimensionality of $d < K$ and a bottleneck layer of dimension $d_{\mathrm{bottleneck}} < d$.

To measure SAE quality, we monitor the explained variance ($R^2$) in latent space, as well as the distance of both the dictionary and codes to ground truth using a standard matching algorithm to control for permutation invariance of the codes~\citep{marks2024enhancingneuralnetworkinterpretability,hindupur2025projecting,dhimoila2026cross}. To measure alignment quality, we use the $R^2$ of the translation and concept consistency properties, as well as the distance between the pushforward measures in both concept space and latent space using our losses. See \cref{app:training_details} for details on the SAE architecture and training procedure.

\subsection{Hyperparameters and Training Details}
\label{app:training_details}

\paragraph{SAE architecture.} 
All variants of the SAE framework rely on specific hidden assumptions on the underlying data-generating process, on the shape of receptive fields~\cite{hindupur2025projecting}, and on the type of atoms to be discovered \citep{fel2025archetypal,dhimoila2026cross}. The purpose of this work is not to show which of these particular methods and assumptions best fit empirical latent spaces, but rather to focus on concept alignment properties. Thus, throughout this work, we adopt the \emph{batchtopk} architecture \citep{bussmann2024batchtopk} as implemented by \emph{overcomplete} \citep{fel2025archetypal} for all SAEs across all experiments, with a dictionary size of $K = 8d$ and a sparsity level of $k = \frac{d}{20}$. These constants also hold for the synthetic setting, where we take $d = 384$, $K = 3072$ and $k = 19$.

\paragraph{Training.} In the vision setting, we train all SAEs on $N=1.2 \times 10^6$ unique samples from ImageNet \citep{imagenet2009}, with a batch size of $b=1024$ and a total number of tokens seen during training of $T=5 \times 10^7$. This corresponds to approximately $E = 40$ epochs. During hyperparameter tuning, we limit the number of epochs to $E = 4$. The learning rate ($\sim 10^{-3}$) is calibrated such that the $\Ldcy$ only run converges to $1.1$ times its final value in $E = 4$ epochs. We keep the same learning rate for all other runs, which all converge well within $E = 40$ epochs. We do not observe significant variance across runs with different seeds.

Our results seem not too sensitive to the regularization strengths $\alpha_{\mathrm{tr/cont/...}}$, which are calibrated as follows. Let $\mathcal{L}^*$ be the final $\Ldcy$ loss after training on $5\times10^7$ tokens with no regularisation. Then, for $\mathrm{reg} \in \{\cdot_\mathrm{tr}, \cdot_\mathrm{cont}, \cdot_\mathrm{tr}^\mathrm{dist}, \cdot_\mathrm{cont}^\mathrm{dist}\}$, we sweep over $\alpha_\mathrm{reg}$ and keep the maximum order of magnitude such that the final $\Ldcy$ term of the loss is less than $0.9\times\mathcal{L}^*$. All subsequent experiments use these fixed values for $\alpha_\mathrm{reg}$ if the corresponding regularisation is included, and $0$ otherwise. This procedure yields $\alpha_\mathrm{tr}=\alpha_\mathrm{cont}=1$, and $\alpha_\mathrm{tr}^\mathrm{dist}=\alpha_\mathrm{cont}^\mathrm{dist}=10^3$. We find that results seem to remain consistent with $\alpha$ varying by one or two orders of magnitude in most cases.

In the language case, SAEs are trained on a total of $5\times10^7$ unique tokens extracted from the 
{\verb "HuggingFaceFW/fineweb"} dataset, and aligned across tokenizers using the same procedure as \citep{jiralerspong2026cross}. For the models, we use {\verb "google/gemma-2-9b"} and {\verb "meta-llama/Llama-3.1-8B"}. All hyperparameter tuning details are the same as in the vision setting.

In the multimodal setting, we sub-sample LAION-400M \citep{schuhmann2021laion} to get $N\sim10^6$ unique image-text pairs, and keep the same other parameters as for ImageNet. We select CLIP ({\verb "openai/clip-vit-base-patch32"}) \citep{radford2021} and OpenCLIP-L ({\verb "laion/CLIP-ViT-L-14-laion2B-s32B-b82K"}) \citep{cherti2023reproducible}. These models are used to extract embeddings for both images and captions. Here, we focus on SAEs coupled to the vision and text towers of the same model, as the scope of this setting is cross-modal rather than cross-model alignment. 
In the synthetic setting, we train all SAEs on $N=10^6$ unique samples, with a batch size of $b=10^3$ and a total number of tokens seen during training of $T=10^7$, which corresponds to $E = 10$ epochs.

\subsection{Mixed Training Regime}
\label{app:mixed}

\paragraph{} In the mixed training regime, we split the dataset into two subsets. Let $p \in [0, 1]$ be the proportion of matched pairs. We randomly select $p \cdot N$ pairs to be in the matched subset $X^i_A$, and the rest to be in the unmatched subset $X^i_B$. The token datasets are then the concatenated $X^i = X^i_A \cup X^i_B$. At this point, in all settings, all domains are perfectly aligned. We then shuffle the unmatched subsets $X^i_B$ independently to break alignment. We keep track of matched and unmatched pairs using binary labels $\bm l \in \{0, 1\}^N$, where $\bm l_n = 1$ if the $n$-th pair is matched and $0$ otherwise. During training, batches are randomly sampled from the whole dataset. The distributional regularization terms $\Ltrdist$ and $\Lcontdist$ are computed on the whole batch. The instance-wise regularization terms $\Ltr$ and $\Lcont$ are computed only on the matched pairs, i.e. those for which $\bm l_n = 1$.

\subsection{Compute}
\label{app:compute}

\paragraph{}All final experiment runs were run on a single H100 GPU. It takes about $40$min to train a CoSAE on $50$M precomputed CLS tokens. It takes about $7$h to train a CoSAE on $50$M tokens extracted on the fly from LLMs using activation buffers. Most experiments were run on an RTX3060 for early testing, where training a similar CoSAE on $5$M precomputed CLS tokens takes a few minutes.

\section{Controlled synthetic evaluation of shared and private factor recovery}
\label{app:synthetic_sharedness}

\label{app:synthetic_sharedness_motivation}
Alignment losses encourage two representations to agree, but agreement does not by itself show that they recover the same generating factor. Independently trained sparse autoencoders (SAEs) may place the same factor at different code
indices, while dependent code coordinates may fail to isolate any one factor. We therefore use a controlled two-view generator to distinguish within-view factor recovery from cross-view correspondence. By varying the amount of shared
structure, the experiment also tests whether alignment preserves private factors and whether equal factor marginals cause unsupported declarations of sharedness.

\subsection{Design and evaluation}
\label{app:synthetic_sharedness_protocol}

For each view $v\in\{1,2\}$, observations in $\mathbb{R}^{64}$ are generated as
\begin{equation}
\widetilde{x}^{(v)}
=A_{\mathrm{s}}^{(v)}s+A_{\mathrm{p}}^{(v)}p^{(v)}+\varepsilon^{(v)},
\qquad \varepsilon^{(v)}\sim\mathcal{N}(0,0.02^2I).
\label{eq:synthetic_shared_private_generator}
\end{equation}
The coefficient vector $s$ is shared exactly across views, whereas $p^{(1)}$ and $p^{(2)}$ are independent. Dictionary columns are independent normalized Gaussian vectors. Within each block, a fixed-size support is drawn uniformly, and active amplitudes follow $\exp(\mathcal{N}(0,0.5^2))$. Thus every view has 32 factors and four active coefficients per example, while the number of shared factors changes across the three regimes in Table~\ref{tab:synthetic_sharedness_regimes}. In the no-shared regime, private coefficients have the same marginal law but independent realizations.

\begin{table}[t]
\centering
\caption{Shared/private generator. $K_{\mathrm{s}}$ and $K_{\mathrm{p}}$ are the numbers of shared and private factors, with $K_{\mathrm{p}}$ specified per view; $k_{\mathrm{s}}$ and $k_{\mathrm{p}}$ are their active support sizes.
Each SAE has $M=K_{\mathrm{s}}+2K_{\mathrm{p}}$ slots.}
\label{tab:synthetic_sharedness_regimes}
\begin{tabular}{lrrrrr}
\toprule
Regime & $K_{\mathrm{s}}$ & $K_{\mathrm{p}}$ & $k_{\mathrm{s}}$ & $k_{\mathrm{p}}$ & $M$ \\
\midrule
Fully shared & 32 & 0 & 4 & 0 & 32 \\
Partially shared & 8 & 24 & 1 & 3 & 56 \\
No shared factors & 0 & 32 & 0 & 4 & 64 \\
\bottomrule
\end{tabular}
\end{table}

We generate 50,000 training, 10,000 validation, and 20,000 test examples and standardize each coordinate using training statistics. Each view has a separate affine encoder and decoder. BatchTopK~\citep{bussmann2024batchtopk} retains the
largest $4b$ rectified activations in a batch of size $b$; the slot counts in Table~\ref{tab:synthetic_sharedness_regimes} leave capacity for both shared and private factors. Models are trained with AdamW~\citep{loshchilov2019decoupled} for at most 30 epochs, using batches of 1024, an initial learning rate of
$10^{-3}$, cosine decay, and early stopping.

We compare reconstruction alone with four coupled configurations: instance consistency (coefficient $0.1$), instance translation ($1$), distributional translation and consistency ($100$ and $30$), and the full combination.
Reconstruction is normalized by target energy, translation compares unit-norm observations, and consistency is normalized by the total code energy.
Distributional translation compares random one-dimensional projections, while distributional consistency compares code coordinates using a weighted empirical characteristic-function discrepancy~\citep{epps1983test,balestriero2025lejepa}; this is a sliced maximum mean discrepancy~\citep{gretton2012kernel}. All configurations also use a dead-feature auxiliary coefficient of $10^{-3}$. The five seeds vary both the generated dictionaries and optimization, and objective configurations use matched seed identities within each regime. For results quoted with confidence intervals, we treat the five seeded experiments as independent replicates and compute a two-sided 95\% Student's \(t\) confidence interval for their mean. These intervals are not adjusted for the number of reported comparisons.

\paragraph{Ground-truth metrics.}
\label{app:synthetic_sharedness_metrics}
Within each view, a one-to-one Hungarian assignment~\citep{kuhn1955hungarian} matches learned slots to generating factors by absolute test-set correlation; a factor is recovered when its assigned correlation is at least $0.8$.
\emph{Same-index shared recall} additionally requires both views to assign the same slot to the same shared factor above this threshold.

\paragraph{Choice of multiplicity correction.}
We use step-down maxT to identify same-index code coordinates whose absolute cross-view correlation exceeds that expected after breaking example-level correspondence. A single row permutation is applied jointly to all coordinates, preserving the dependence among slotwise statistics under the permutation
reference. After ordering the observed correlations, each is compared with the permutation maximum over the hypotheses remaining at that step \citep{westfall1993resampling}. This can be less conservative than single-step maxT, although the first step is the same global-maximum comparison. We use 10,000 permutations and level $0.05$, apply the plus-one Monte Carlo correction \citep{phipson2010permutation}, and retain single-step results as a sensitivity analysis. A declared slot is called unsupported when it fails the ground-truth same-factor criterion. The procedure calibrates the joint row-independence null within a trained model pair; it does not establish factor identity or, without additional subset-pivotality assumptions, strong familywise control for semantic
hypotheses under partial sharing.

\paragraph{Routing-reference excess estimators.}
We perturb test observations along known unit-norm generating directions, re-encode them, and cross-decode the code change. Shared responses are compared with the corresponding target direction. For each private source factor, RMS routing excess subtracts the mean normalized RMS magnitude under 5,000
fixed-point-free target-slot permutations from the same-index response, holding the source code change fixed. Strongest-target excess instead projects the response onto the ground-truth target dictionary, selects the largest target-factor RMS coordinate, and reselects that maximum within every reference
permutation before subtraction. We average both factor-level excesses across private factors and translation directions. The first estimator tests for preferential total response magnitude; the second tests for preferential response along some ground-truth target factor. Neither measures absolute
response, and the derangement distribution is a descriptive routing reference, not a calibrated semantic randomization null. We therefore also report absolute responses: a negative excess means below-reference routing, not negative causal
transfer, while a near-zero excess does not establish equivalence or absent target-factor response.

\subsection{Results}
\label{app:synthetic_sharedness_results}

\begin{table}[t]
\centering
\scriptsize
\caption{Recovery results, reported as mean $\pm$ sample SD across five seeds. Reconstruction and private recall are averaged across views. Shared recall requires ground-truth recovery at the same code index; unsupported is the number
of statistically declared slots that fail this criterion. A dash denotes an inapplicable endpoint.}
\label{tab:synthetic_sharedness_primary}
\resizebox{\linewidth}{!}{
\begin{tabular}{llrrrr}
\toprule
Regime & Objective & Mean $R^2$ & Private recall & Shared recall & Unsupported \\
\midrule
Fully shared & Reconstruction & $0.984\pm0.008$ & --- & $0.025\pm0.026$ & $25.800\pm3.114$ \\
Fully shared & Instance consistency & $0.983\pm0.003$ & --- & $0.075\pm0.036$ & $6.800\pm1.483$ \\
Fully shared & Instance translation & $0.978\pm0.013$ & --- & $0.975\pm0.034$ & $0.800\pm1.095$ \\
Fully shared & Distributional & $0.862\pm0.012$ & --- & $0.025\pm0.014$ & $23.200\pm1.924$ \\
Fully shared & Full & $0.963\pm0.018$ & --- & $0.950\pm0.047$ & $1.400\pm1.140$ \\
\midrule
Partially shared & Reconstruction & $0.994\pm0.001$ & $1.000\pm0.000$ & $0.025\pm0.056$ & $1.600\pm0.894$ \\
Partially shared & Instance consistency & $0.994\pm0.000$ & $1.000\pm0.000$ & $0.100\pm0.105$ & $1.600\pm1.949$ \\
Partially shared & Instance translation & $0.847\pm0.005$ & $0.988\pm0.011$ & $0.975\pm0.056$ & $1.800\pm1.483$ \\
Partially shared & Distributional & $0.954\pm0.008$ & $1.000\pm0.000$ & $0.050\pm0.068$ & $4.400\pm2.966$ \\
Partially shared & Full & $0.809\pm0.006$ & $0.988\pm0.011$ & $0.975\pm0.056$ & $2.200\pm1.924$ \\
\midrule
No shared & Reconstruction & $0.994\pm0.000$ & $1.000\pm0.000$ & --- & $0.200\pm0.447$ \\
No shared & Instance consistency & $0.994\pm0.000$ & $1.000\pm0.000$ & --- & $0.000\pm0.000$ \\
No shared & Instance translation & $0.981\pm0.001$ & $1.000\pm0.000$ & --- & $0.000\pm0.000$ \\
No shared & Distributional & $0.953\pm0.004$ & $1.000\pm0.000$ & --- & $0.200\pm0.447$ \\
No shared & Full & $0.922\pm0.009$ & $0.997\pm0.007$ & --- & $0.000\pm0.000$ \\
\bottomrule
\end{tabular}}
\end{table}

Reconstruction-only training shows that the generating factors are recoverable: it recovers all private factors in the partially shared and no-shared regimes, with mean reconstruction $R^2=0.994$. Nevertheless, same-index shared recall is only $0.025$ in both regimes containing shared factors (Table~\ref{tab:synthetic_sharedness_primary}). Accurate within-view recovery therefore does not align code indices across views.

Instance translation raises same-index shared recall to $0.975$ in both shared regimes. Its reconstruction remains $0.978$ when all factors are shared but falls to $0.847$ under partial sharing, while private recall remains $0.988$. The full
objective similarly reaches shared recall of $0.950$ and $0.975$ in the fully and partially shared regimes, respectively, with partial-sharing private recall of $0.988$. These comparisons show that shared correspondence can coexist with private recovery, although the reconstruction cost depends on the objective and the amount of private structure.

Distributional training gives the complementary dissociation. Under partial sharing it preserves reconstruction ($R^2=0.954$) and all private factors, but same-index shared recall remains $0.050$. Matching the tested marginals is thus
insufficient to identify which factors correspond. In the no-shared regime, the full and both instance-only configurations declare no shared slots across the five seeds while retaining private recall of at least $0.997$.

The interventions further separate routing magnitude from destination. Under the full objective, mean shared-direction cosine is $0.370$ in the fully shared regime and $0.402$ under partial sharing, despite shared recall of at least $0.950$. In the no-shared regime, mean private routing-magnitude excess is
$-0.00014$ (95\% interval $[-0.00133,0.00105]$), whereas the strongest target-factor response is $0.450$ and its reference excess is $0.0351$ ($[0.0208,0.0494]$). A near-zero aggregate routing excess therefore does not establish the absence of a response along a particular target factor.

\subsection{Scope}
\label{app:synthetic_sharedness_conclusion}

The distributional coefficients were selected exploratorily using the same five seed identities reported here, so these comparisons are not an independent confirmation. Moreover, the generator uses linear mixing, independent private blocks, fixed support sizes, and sufficient slots for the union of factors; it does not cover dependent private factors, nonlinear mixing, or strongly overcomplete recovery. Within this controlled setting, the results distinguish four properties that should not be collapsed into one claim of shared geometry: reconstruction, factor recovery, same-index correspondence, and intervention destination.

\section{Proofs}
\label{app:proofs}

\subsection{Alignment Dualities}

\paragraph{}Our first proof is that of the duality between latent and concept alignment, or translation and concept consistency, under the assumption that decoders are injective.

\begin{proposition*}
\label{prop:PtrPcont_proof}
Suppose injectivity of the decoders, and suppose $\Pdcy$. Then
\[
    \Ptr \Leftrightarrow\Pcont
\]
\end{proposition*}

\begin{proof}
\underline{$\Ptr \Leftarrow\Pcont$:} This direction does not require the injectivity assumption. Let $x \in X$ and $i, j \in I$. Suppose $\Pcont$. Then:
\begin{align*}
\Pcont : \psi_i (\xvec_i) = \psi_j (\xvec_j) &\Rightarrow \psi^{-1}_j ( \psi_i (\xvec_i)) = \psi^{-1}_j ( \psi_j (\xvec_j))\\
&\Rightarrow \psi^{-1}_j ( \psi_i (\xvec_i)) = \xvec_j: \Ptr& \text{by $\Pdcy$}
\end{align*}

\underline{$\Ptr \Rightarrow\Pcont$:} Let $x \in X$ and $i, j \in I$. Suppose $\Ptr$. Then:
\begin{align*}
\Ptr : (\psi^{-1}_j \circ \psi_i) (\xvec_i) = x_j &\Rightarrow \psi^{-1}_j ( \psi_i (\xvec_i)) = \psi^{-1}_j ( \psi_j (\xvec_j)) & \text{by $\Pdcy$}\\
&\Rightarrow \psi_i (\xvec_i) = \psi_j (\xvec_j) : \Pcont & \text{by injectivity of $\psi_j^{-1}$}
\end{align*}
\end{proof}

\paragraph{}Then, we prove in \Cref{prop:PtrPtrdist} the trivial result that pair-wise alignment implies distributional alignment.

\begin{proposition*}
\label{prop:PtrPtrdist_proof} $\Ptr \Rightarrow \Ptrdist \quad \land \quad \Pcont \Rightarrow \Pcontdist$
\end{proposition*}

\begin{proof}
\underline{$\Ptr \Rightarrow \Ptrdist$:} Assume $\Ptr$, i.e., for all samples $x \in X$ and domains $i, j \in I$,
\[
(\psi_j^{-1} \circ \psi_i)(\xvec_i) = \xvec_j,
\quad \text{where } \xvec_i = \phi_i(x), \; \xvec_j = \phi_j(x).
\]
Let $A \subseteq X^j$ be measurable. Then, by definition of the pushforward measure,
\begin{align*}
\big((\psi_j^{-1} \circ \psi_i)\# \mu_i\big)(A)
&= \mu_i\big((\psi_j^{-1} \circ \psi_i)^{-1}(A)\big) \\
&= \mu\big(\phi_i^{-1}((\psi_j^{-1} \circ \psi_i)^{-1}(A))\big).
\end{align*}
For $x \in X$, using $\Ptr$,
\[
x \in \phi_i^{-1}((\psi_j^{-1} \circ \psi_i)^{-1}(A))
\;\Leftrightarrow\;
(\psi_j^{-1} \circ \psi_i)(\phi_i(x)) \in A
\;\Leftrightarrow\;
\phi_j(x) \in A.
\]
Hence
\[
\phi_i^{-1}((\psi_j^{-1} \circ \psi_i)^{-1}(A)) = \phi_j^{-1}(A),
\]
and therefore
\begin{align*}
\big((\psi_j^{-1} \circ \psi_i)\# \mu_i\big)(A)
&= \mu(\phi_j^{-1}(A)) \\
&= \mu_j(A).
\end{align*}
Thus $(\psi_j^{-1} \circ \psi_i)\# \mu_i = \mu_j$.

\underline{$\Pcont \Rightarrow \Pcontdist$:} Assume $\Pcont$, i.e., for all samples $x \in X$ and domains $i, j \in I$,
\[
\psi_i(\xvec_i) = \psi_j(\xvec_j),
\quad \text{where } \xvec_i = \phi_i(x), \; \xvec_j = \phi_j(x).
\]
Let $B \subseteq \R^K$ be measurable. Then
\begin{align*}
(\psi_i \# \mu_i)(B)
&= \mu_i(\psi_i^{-1}(B)) \\
&= \mu(\phi_i^{-1}(\psi_i^{-1}(B))).
\end{align*}
For $x \in X$, using $\Pcont$,
\[
x \in \phi_i^{-1}(\psi_i^{-1}(B))
\;\Leftrightarrow\;
\psi_i(\phi_i(x)) \in B
\;\Leftrightarrow\;
\psi_j(\phi_j(x)) \in B
\;\Leftrightarrow\;
x \in \phi_j^{-1}(\psi_j^{-1}(B)).
\]
Thus
\[
\phi_i^{-1}(\psi_i^{-1}(B)) = \phi_j^{-1}(\psi_j^{-1}(B)),
\]
and therefore
\begin{align*}
(\psi_i \# \mu_i)(B)
&= \mu(\phi_j^{-1}(\psi_j^{-1}(B))) \\
&= \mu_j(\psi_j^{-1}(B)) \\
&= (\psi_j \# \mu_j)(B).
\end{align*}
Hence $\psi_i \# \mu_i = \psi_j \# \mu_j$.
\end{proof}

\subsection{Our distribution losses are MMDs}

\begin{proposition}
\label{prop:EP_is_MMD}
Let $\mu,\nu$ be probability measures on $\R^d$, and let $w : \R \to \R_+$ be an integrable weighting function. For $\uvec \in \mathcal{S}^{d-1}$, define
\[
\mathrm{EP}_{\uvec}(\mu,\nu)
\coloneqq \int_{-\infty}^{\infty}
\left|
\varphi_{\mu,\uvec}(t) - \varphi_{\nu,\uvec}(t)
\right|^2 w(t)\,dt,
\]
where $\varphi_{\mu,\uvec}(t) \coloneqq \E_{x \sim \mu} e^{it\langle \xvec,\uvec\rangle}$ is the characteristic function of the one-dimensional projection of $\mu$ along $\uvec$.
Then $\mathrm{EP}_{\uvec}(\mu,\nu)$ is a Maximum Mean Discrepancy:
\[
\mathrm{EP}_{\uvec}(\mu,\nu)
= \mathrm{MMD}^2_{k_{\uvec}}(\mu,\nu),
\]
with respect to the positive definite kernel
\[
k_{\uvec}(x,y)
\coloneqq \int_{-\infty}^{\infty}
e^{it\langle \xvec-\yvec,\uvec\rangle} w(t)\,dt.
\]
\end{proposition}

\begin{proof}
We expand the squared difference:
\begin{align*}
\mathrm{EP}_{\uvec}(\mu,\nu)
&= \int \left|
\varphi_{\mu,\uvec}(t) - \varphi_{\nu,\uvec}(t)
\right|^2 w(t)\,dt \\
&= \int \left(
\varphi_{\mu,\uvec}(t)\overline{\varphi_{\mu,\uvec}(t)}
+ \varphi_{\nu,\uvec}(t)\overline{\varphi_{\nu,\uvec}(t)}
- 2 \Re\big(\varphi_{\mu,\uvec}(t)\overline{\varphi_{\nu,\uvec}(t)}\big)
\right) w(t)\,dt.
\end{align*}
Using independence of samples and Fubini's theorem,
\begin{align*}
\int \varphi_{\mu,\uvec}(t)\overline{\varphi_{\mu,\uvec}(t)} w(t)\,dt
&= \E_{\xvec,\xvec' \sim \mu}
\int e^{it\langle \xvec-\xvec',\uvec\rangle} w(t)\,dt \\
&= \E_{\xvec,\xvec' \sim \mu} k_{\uvec}(\xvec,\xvec').
\end{align*}
Similarly,
\[
\int \varphi_{\nu,\uvec}(t)\overline{\varphi_{\nu,\uvec}(t)} w(t)\,dt
= \E_{\yvec,\yvec' \sim \nu} k_{\uvec}(\yvec,\yvec'),
\]
and
\[
\int \varphi_{\mu,\uvec}(t)\overline{\varphi_{\nu,\uvec}(t)} w(t)\,dt
= \E_{\xvec \sim \mu,\, \yvec \sim \nu} k_{\uvec}(\xvec,\yvec).
\]
Therefore,
\[
\mathrm{EP}_{\uvec}(\mu,\nu)
= \E_{\xvec,\xvec' \sim \mu} k_{\uvec}(\xvec,\xvec')
+ \E_{\yvec,\yvec' \sim \nu} k_{\uvec}(\yvec,\yvec')
- 2 \E_{\xvec \sim \mu,\, \yvec \sim \nu} k_{\uvec}(\xvec,\yvec),
\]
which is exactly $\mathrm{MMD}^2_{k_{\uvec}}(\mu,\nu)$.

Finally, since $k_{\uvec}$ is the Fourier transform of the nonnegative function $w$, it is positive definite by Bochner's theorem.
\end{proof}

\subsection{Linear Alignment}

This section contains the analytical solutions in the linear regression case, along with the corresponding proofs.

\subsubsection{Reconstruction}
\label{app:proof_pca}

Recall \Cref{lem:pca}:

\begin{lemma*}
Consider the self-reconstruction objectives
\begin{equation}
\Ldcy^x = \frac{1}{n} \|\Xmat - \Xmat \Wmat_x \Vmat_x\|_F^2,
\qquad
\Ldcy^y = \frac{1}{n} \|\Ymat - \Ymat \Wmat_y \Vmat_y\|_F^2.
\end{equation}
Then, for each domain independently, any minimizer satisfies:
\begin{itemize}
    \item The encoder $\Wmat_x$ (resp.\ $\Wmat_y$) spans the top-$K$ principal subspace of $\Sigmamat_{xx}$ (resp.\ $\Sigmamat_{yy}$).
    \item The optimal decoder is given by $\Vmat_x = \Wmat_x^\top$ (resp.\ $\Vmat_y = \Wmat_y^\top$) up to a change of basis.
\end{itemize}
\end{lemma*}

\begin{proof}
We treat the $\Xmat$ domain. The $\Ymat$ domain is identical. We omit the $x$ sub- and superscripts for ease of notation, and rewrite the objective
\[
\min_{\Wmat, \Vmat} \frac{1}{n} \|\Xmat - \Xmat \Wmat \Vmat\|_F^2
\]
as
\[
\min_{\rank(\Mmat) \le K} \|\Xmat - \Mmat\|_F^2,
\]
where $\rank(\Mmat)$ denotes the rank of $\Mmat$. Let $\Xmat = \Umat \Sigmamat \Vmat^\top$ be the singular value decomposition of $\Xmat$, where $\Umat \in \R^{n \times n}$ and $\Vmat \in \R^{d_x \times d_x}$ are orthogonal, and $\Sigmamat \in \R^{n \times d_x}$ is diagonal with nonnegative entries. Then, by the Eckart-Young theorem, the optimal solution is given by the truncated SVD of $\Xmat$:
\[
\Mmat = \Umat_K \Sigmamat_K \Vmat_K^\top,
\]
where $\Umat_K \in \R^{n \times K}$ and $\Vmat_K \in \R^{d_x \times K}$ are the matrices of the top-$K$ left and right singular vectors, and $\Sigmamat_K \in \R^{K \times K}$ is the diagonal matrix of the top-$K$ singular values. Therefore, any optimal factorization verifies
\[
\Xmat \Wmat \Vmat = \Umat_K \Sigmamat_K \Vmat_K^\top.
\]
The general solution is given by $\Wmat = \Vmat_K \Rmat$ and $\Vmat = \Rmat^{-1} \Umat_K^\top$ for any invertible matrix $\Rmat \in \R^{K \times K}$.
\end{proof}

\subsubsection{Concept Consistency}
\label{app:proof_cca}

\begin{lemma*}[Degeneracy of concept consistency]
In the absence of additional constraints, minimizing $\Lcont$ is ill-posed: for any pair $(\Wmat_x, \Wmat_y)$, and any orthogonal matrix $\Rmat \in \R^{K \times K}$, the transformation
\begin{equation}
\Wmat_x \rightarrow \Wmat_x \Rmat, \qquad
\Wmat_y \rightarrow \Wmat_y \Rmat
\end{equation}
leaves the objective unchanged. Moreover, the objective admits degenerate solutions (e.g., collapse to zero).
\end{lemma*}

\begin{proof}\textbf{Invariance under change of basis.}
Let $\Rmat \in \R^{K \times K}$ be orthogonal, and define
\[
\widetilde{\Wmat}_x = \Wmat_x \Rmat, 
\qquad
\widetilde{\Wmat}_y = \Wmat_y \Rmat.
\]
Then the corresponding latent representations satisfy
\[
\widetilde{\Zmat}_x = \Xmat \widetilde{\Wmat}_x = \Zmat_x \Rmat,
\qquad
\widetilde{\Zmat}_y = \Ymat \widetilde{\Wmat}_y = \Zmat_y \Rmat.
\]

Therefore,
\begin{equation}
\begin{split}
\|\widetilde{\Zmat}_x - \widetilde{\Zmat}_y\|_F^2
&= \|(\Zmat_x - \Zmat_y)\Rmat\|_F^2 \\
&= \Tr\!\left(\Rmat^\top (\Zmat_x - \Zmat_y)^\top (\Zmat_x - \Zmat_y)\Rmat\right) \\
&= \Tr\!\left((\Zmat_x - \Zmat_y)^\top (\Zmat_x - \Zmat_y)\right) \\
&= \|\Zmat_x - \Zmat_y\|_F^2.
\end{split}
\end{equation}

\paragraph{Degeneracy.}
Consider $\Wmat_x = 0$ and $\Wmat_y = 0$. Then $\Zmat_x = \Zmat_y = 0$, and
\begin{equation}
\Lcont = 0,
\end{equation}
which is the global minimum. More generally, any pair $(\Wmat_x, \Wmat_y)$ such that $\Xmat \Wmat_x = \Ymat \Wmat_y$ achieves zero loss, regardless of the structure or rank of the representations.

\paragraph{Conclusion.}
The objective admits infinitely many minimizers due to invariance under invertible transformations, and includes degenerate solutions such as collapse to zero. Therefore, it is ill-posed without additional constraints.
\end{proof}

\begin{lemma*}[CCA from concept consistency with whitening]
Consider minimizing $\Lcont$ under the whitening constraints
\begin{equation}
\frac{1}{n} \Zmat_x^\top \Zmat_x = \Imat_K, \qquad
\frac{1}{n} \Zmat_y^\top \Zmat_y = \Imat_K.
\end{equation}
Then the optimal encoders $(\Wmat_x, \Wmat_y)$ are given by the top-$K$ canonical directions between $\Xmat$ and $\Ymat$, i.e., they solve the canonical correlation analysis (CCA) problem.
\end{lemma*}

\begin{proof}
Recall that
\begin{equation}
\Lcont 
= \frac{1}{n} \|\Xmat \Wmat_x - \Ymat \Wmat_y\|_F^2.
\end{equation}
Expanding the Frobenius norm yields
\begin{align}
\Lcont
&= \Tr(\Wmat_x^\top \Sigmamat_{xx} \Wmat_x)
+ \Tr(\Wmat_y^\top \Sigmamat_{yy} \Wmat_y)
- 2\,\Tr(\Wmat_x^\top \Sigmamat_{xy} \Wmat_y).
\end{align}

Under the whitening constraints
\begin{equation}
\Wmat_x^\top \Sigmamat_{xx} \Wmat_x = \Imat_K,
\qquad
\Wmat_y^\top \Sigmamat_{yy} \Wmat_y = \Imat_K,
\end{equation}
the first two terms are constant and equal to $K$. Therefore, minimizing $\mathcal{L}_{\mathrm{cont}}$ is equivalent to maximizing
\begin{equation}
\Tr(\Wmat_x^\top \Sigmamat_{xy} \Wmat_y).
\end{equation}
subject to the whitening constraints. This is exactly the CCA problem, whose solution is given by the top-$K$ canonical directions between $\Xmat$ and $\Ymat$.

\paragraph{Closed-form solution.} Let $\widetilde{\Wmat}_x = \Sigmamat_{xx}^{\frac{1}{2}} \Wmat_x$ and $\widetilde{\Wmat}_y = \Sigmamat_{yy}^{\frac{1}{2}} \Wmat_y$. Then the whitening constraints become $\widetilde{\Wmat}_x^\top \widetilde{\Wmat}_x = \Imat_K$ and $\widetilde{\Wmat}_y^\top \widetilde{\Wmat}_y = \Imat_K$, and the objective becomes
\begin{equation}
\Tr(\widetilde{\Wmat}_x^\top \Mmat \widetilde{\Wmat}_y),
\end{equation}
where $\Mmat = \Sigmamat_{xx}^{-\frac{1}{2}} \Sigmamat_{xy} \Sigmamat_{yy}^{-\frac{1}{2}}$. Let $\Mmat = \Umat \Sigmamat \Vmat^\top$ be the singular value decomposition of $\Mmat$. Then the optimal solution is given by $\widetilde{\Wmat}_x = \Umat_K$ and $\widetilde{\Wmat}_y = \Vmat_K$, where $\Umat_K$ and $\Vmat_K$ are the matrices of the top-$K$ left and right singular vectors of $\Mmat$. Therefore, the optimal encoders are given by $\Wmat_x = \Sigmamat_{xx}^{-\frac{1}{2}} \Umat_K$ and $\Wmat_y = \Sigmamat_{yy}^{-\frac{1}{2}} \Vmat_K$ up to an arbitrary orthogonal transform $\Rmat \in \R^{K \times K}$.
\end{proof}

\begin{lemma*}[Concept consistency with reconstruction]
\label{lem:cont_dcy_proof}
Consider the objective
\begin{equation}
\mathcal{L}(\Wmat_x, \Wmat_y, \Vmat_x, \Vmat_y)
=
\Lcont
+ \lambda \left(
\Ldcy^x
+
\Ldcy^y
\right).
\end{equation}

\textbf{(i) General case.}
After optimizing over $(\Vmat_x, \Vmat_y)$, the objective depends only on the subspaces spanned by $(\Wmat_x, \Wmat_y)$ and combines:
\begin{itemize}
    \item a consistency term favoring alignment between $\Xmat \Wmat_x$ and $\Ymat \Wmat_y$,
    \item reconstruction terms favoring principal subspaces of $\Sigmamat_{xx}$ and $\Sigmamat_{yy}$.
\end{itemize}
However, the problem is ill-posed due to scaling degeneracies in the consistency term.

\textbf{(ii) Whitening.}
If we impose the whitening constraints
\begin{equation}
\Wmat_x^\top \Sigmamat_{xx} \Wmat_x = \Imat_K, \qquad
\Wmat_y^\top \Sigmamat_{yy} \Wmat_y = \Imat_K,
\end{equation}
then:
\begin{itemize}
    \item the reconstruction terms become constant and do not influence the encoders,
    \item the objective reduces to $\max \Tr(\Wmat_x^\top \Sigmamat_{xy} \Wmat_y)$, whose solution is given by the top-$K$ canonical directions (CCA),
    \item the optimal decoders are uniquely determined (up to invariances) as the corresponding least-squares readouts.
\end{itemize}
In particular, under whitening, the solution is independent of $\lambda$.
\end{lemma*}

\begin{proof}
We eliminate the decoders and derive the reduced objective.

\paragraph{Step 1: Optimal decoders.}
For fixed encoders, the reconstruction terms are independent least squares problems. As in the PCA case, the optimal decoders are
\begin{equation}
\Vmat_x^* = (\Wmat_x^\top \Sigmamat_{xx} \Wmat_x)^{-1} \Wmat_x^\top \Sigmamat_{xx},
\qquad
\Vmat_y^* = (\Wmat_y^\top \Sigmamat_{yy} \Wmat_y)^{-1} \Wmat_y^\top \Sigmamat_{yy},
\end{equation}
assuming invertibility.

Plugging back, the reconstructions correspond to projections onto the column spaces of $\Xmat \Wmat_x$ and $\Ymat \Wmat_y$, and the losses reduce to
\begin{equation}
\frac{1}{n} \|\Xmat - \widehat{\Xmat}\|_F^2
= \Tr(\Sigmamat_{xx}) - \Tr(\Pmat_x \Sigmamat_{xx}),
\quad
\frac{1}{n} \|\Ymat - \widehat{\Ymat}\|_F^2
= \Tr(\Sigmamat_{yy}) - \Tr(\Pmat_y \Sigmamat_{yy}),
\end{equation}
where
\begin{equation}
\Pmat_x = \Wmat_x (\Wmat_x^\top \Sigmamat_{xx} \Wmat_x)^{-1} \Wmat_x^\top \Sigmamat_{xx},
\quad
\Pmat_y = \Wmat_y (\Wmat_y^\top \Sigmamat_{yy} \Wmat_y)^{-1} \Wmat_y^\top \Sigmamat_{yy}.
\end{equation}

\paragraph{Step 2: Reduced objective.}
Up to constants independent of $(\Wmat_x, \Wmat_y)$, the objective becomes
\begin{equation}
\mathcal{L}_{\mathrm{red}}(\Wmat_x, \Wmat_y)
=
\frac{1}{n} \|\Xmat \Wmat_x - \Ymat \Wmat_y\|_F^2
- \lambda \left(
\Tr(\Pmat_x \Sigmamat_{xx})
+
\Tr(\Pmat_y \Sigmamat_{yy})
\right).
\end{equation}

Expanding the first term,
\begin{equation}
\frac{1}{n} \|\Xmat \Wmat_x - \Ymat \Wmat_y\|_F^2
=
\Tr(\Wmat_x^\top \Sigmamat_{xx} \Wmat_x)
+
\Tr(\Wmat_y^\top \Sigmamat_{yy} \Wmat_y)
- 2\,\Tr(\Wmat_x^\top \Sigmamat_{xy} \Wmat_y).
\end{equation}

\paragraph{Step 3: Variational characterization.}
We characterize the optimal subspaces without differentiating the reduced objective.

First, observe that the reconstruction terms depend only on the projection matrices $\Pmat_x$ and $\Pmat_y$, and therefore only on the column spaces of $\Wmat_x$ and $\Wmat_y$. As in the PCA case, maximizing
\[
\Tr(\Pmat_x \Sigmamat_{xx})
\]
selects the top-$K$ eigenspace of $\Sigmamat_{xx}$, and similarly for $\Sigmamat_{yy}$.

On the other hand, similarly to the pure CCA case, the consistency term
$\|\Xmat \Wmat_x - \Ymat \Wmat_y\|_F^2$
is ill-posed without additional constraints: it is invariant to joint rescaling of $(\Wmat_x, \Wmat_y)$ and admits degenerate solutions (e.g., $\Wmat_x = \Wmat_y = 0$).

To remove this degeneracy, we impose whitening constraints
\[
\Wmat_x^\top \Sigmamat_{xx} \Wmat_x = \Imat_K, \qquad
\Wmat_y^\top \Sigmamat_{yy} \Wmat_y = \Imat_K.
\]
Under these constraints, minimizing the consistency term is equivalent to maximizing
\[
\Tr(\Wmat_x^\top \Sigmamat_{xy} \Wmat_y),
\]
which yields the canonical directions, as in CCA.

Importantly, under whitening, the reconstruction terms become constant:
\[
\Tr(\Pmat_x \Sigmamat_{xx}) = \Tr(\Wmat_x^\top \Sigmamat_{xx} \Wmat_x) = K,
\]
and similarly for $\Ymat$. Therefore, the reconstruction objective no longer influences the solution, regardless of the value of $\lambda$. Its only effect is to identify the decoders.

As a result, the combined objective reduces exactly to CCA, and the optimal encoders are given by the top-$K$ canonical directions between $\Xmat$ and $\Ymat$.
\end{proof}

\subsubsection{Translation}

\begin{lemma*}[Translation and reduced-rank regression]
\label{lem:tr_rrr_proof}
Consider the (one-directional) translation objective
\begin{equation}
\Ltr^{x \to y}(\Wmat_x, \Vmat_y)
=
\frac{1}{n} \|\Ymat - \Xmat \Wmat_x \Vmat_y\|_F^2.
\end{equation}
Then minimizing $\Ltr^{x \to y}$ over $(\Wmat_x, \Vmat_y)$ is equivalent to RRR of $\Ymat$ onto $\Xmat$ with rank constraint $K$. The factorization $(\Wmat_x, \Vmat_y)$ is not identifiable: only the product $\Wmat_x \Vmat_y$ is. One such factorization recovers the CCA.
\end{lemma*}

\begin{proof}
Let $\Bmat = \Wmat_x \Vmat_y$. Then the problem reduces to
\begin{equation}
\min_{\rank(\Bmat) \le K} \frac{1}{n} \|\Ymat - \Xmat \Bmat\|_F^2,
\end{equation}
which is exactly RRR of $\Ymat$ onto $\Xmat$ with rank constraint $K$. The optimal unconstrained $\Bmat$ is given by $\Bmat_{\mathrm{OLS}} = \Sigmamat_{xx}^{-1} \Sigmamat_{xy}$.

We rewrite the problem as
\begin{equation}
\min_{\rank(\Bmat) \le K} \|\Ymat - \Xmat \Bmat_\mathrm{OLS}\|_F^2 + \|\Xmat (\Bmat - \Bmat_\mathrm{OLS})\|_F^2,
\end{equation}
where we used the Pythagorean theorem for least squares problems. The first term is constant. Using the definition of $\Sigmamat_{xx}$, we have
\begin{align*}
\|\Xmat (\Bmat - \Bmat_\mathrm{OLS})\|_F^2
&= n \|\Sigmamat_{xx}^{\frac{1}{2}} (\Bmat - \Bmat_\mathrm{OLS})\|_F^2\\
&= n \|\Cmat - \Sigmamat_{xx}^{-\frac{1}{2}} \Sigmamat_{xy}\|_F^2.
\end{align*}
with $\Cmat \coloneqq \Sigmamat_{xx}^{\frac{1}{2}} \Bmat$. We want to make the CCA appear. The only thing missing is the $\Sigmamat_{yy}^{-\frac{1}{2}}$ term. We introduce a whitening step: $\Sigmamat_{xx}^{-\frac{1}{2}} \Sigmamat_{xy} = \Mmat \Sigmamat_{yy}^{\frac{1}{2}}$, where $\Mmat = \Sigmamat_{xx}^{-\frac{1}{2}} \Sigmamat_{xy} \Sigmamat_{yy}^{-\frac{1}{2}}$, and $\Cmat = \Dmat \Sigmamat_{yy}^{\frac{1}{2}}$. Then, the objective becomes
\begin{equation}
\min_{\rank(\Dmat) \le K} \|\Dmat - \Mmat\|_{F,\Sigmamat_{yy}}^2
\end{equation}
Conveniently, $\Mmat$ is exactly the same as in CCA. We can express the solution in terms of the truncated SVD $\Umat_K \Sigmamat_K \Vmat_K^\top + \epsilon$ of $\Mmat$:
\begin{equation}
\Bmat = \Sigmamat_{xx}^{-\frac{1}{2}} \Umat_K \Sigmamat_K \Vmat_K^\top \Sigmamat_{yy}^{\frac{1}{2}}
\end{equation}
With this decomposition, we can identify a particular factorization $\Wmat_x = \Sigmamat_{xx}^{-\frac{1}{2}} \Umat_K$ and $\Vmat_y = \Sigmamat_K \Vmat_K^\top \Sigmamat_{yy}^{\frac{1}{2}}$. This factorization recovers the CCA solution for the encoders and the corresponding least-squares decoder. This factorization is not unique.
\end{proof}

\begin{lemma*}[Translation with reconstruction]
\label{lem:tr_dcy_proof}
Under the combined objective, the optimal encoders are again given by the CCA, while the decoders are given by an interpolation between the solutions of \Cref{lem:tr_rrr} and \Cref{lem:cont_dcy}.
\end{lemma*}

\begin{proof}
As before, we start by eliminating the decoders given $\Wmat_x$ and $\Wmat_y$. Due to the combined effect of self-reconstruction and translation, we get, assuming invertibility,
\begin{equation}
\Vmat_x^* = (\Wmat_y^\top \Sigmamat_{yy} \Wmat_y + \lambda \Wmat_x^\top \Sigmamat_{xx} \Wmat_x)^{-1} (\Wmat_y^\top \Sigmamat_{xy}^\top + \lambda \Wmat_x^\top \Sigmamat_{xx})
\end{equation}
and similarly for $\Vmat_y^*$.

The next step is to assume whitening. We get
\begin{equation}
\Vmat_x^* = \frac{1}{n(1+\lambda)} \left(\Wmat_y^\top \Sigmamat_{xy}^\top + \lambda \Wmat_x^\top \Sigmamat_{xx}\right),
\end{equation}
and similarly for $\Vmat_y^*$. Plugging back, and thanks to whitening, the objective reduces to
\begin{equation}
\max_{\Wmat_x, \Wmat_y} \Tr(\Wmat_x^\top \Sigmamat_{xy} \Wmat_y) + \lambda \left(\Tr(\Wmat_x^\top \Sigmamat_{xx} \Wmat_x) + \Tr(\Wmat_y^\top \Sigmamat_{yy} \Wmat_y)\right),
\end{equation}
The term in $\lambda$ becomes constant again. We add a whitening step to the $\Wmat$ to make the CCA appear. Let $\widetilde{\Wmat}_x = \Sigmamat_{xx}^{\frac{1}{2}} \Wmat_x$ and $\widetilde{\Wmat}_y = \Sigmamat_{yy}^{\frac{1}{2}} \Wmat_y$, and define $\Mmat = \Sigmamat_{xx}^{-\frac{1}{2}} \Sigmamat_{xy} \Sigmamat_{yy}^{-\frac{1}{2}}$. The objective reduces to
\begin{equation}
\max_{\widetilde{\Wmat}_x, \widetilde{\Wmat}_y} \Tr(\widetilde{\Wmat}_x^\top \Mmat \widetilde{\Wmat}_y),
\end{equation}
and we again get the CCA solution for the encoders. The difference is that now, the decoders interpolate between self-projections and cross-projections, with a weighting controlled by $\lambda$.
\end{proof}

\subsubsection{Cycle consistency}

\begin{lemma}
\label{lem:cycle_proof}
Under the $\Lcy$ objective, the encoders and decoders recover the PCA.
\end{lemma}

\begin{proof}
Let $\Amat \coloneqq \Wmat_x \Vmat_y$ and $\Bmat \coloneqq \Wmat_y \Vmat_x$. Let $\Mmat \coloneqq \Amat \Bmat$ and $\Nmat \coloneqq \Bmat \Amat$. Then the cycle consistency objective can be rewritten as
\begin{equation}
\min_{\rank(\Amat) \le K, \rank(\Bmat) \le K} \frac{1}{n} \|\Xmat - \Xmat \Mmat\|_F^2 + \frac{1}{n} \|\Ymat - \Ymat \Nmat\|_F^2.
\end{equation}
Without the coupling constraint between $\Mmat$ and $\Nmat$, the problem would reduce to two independent PCA problems, whose solutions are given by the top-$K$ principal subspaces of $\Sigmamat_{xx}$ and $\Sigmamat_{yy}$: $\Mmat^* = \Pmat_x \Pmat_x^\top$ and $\Nmat^* = \Pmat_y \Pmat_y^\top$. This gives a lower bound on the objective. We now show that this lower bound is achievable, and therefore optimal. Let $\Rmat \in \R^{K \times K}$ be an arbitrary invertible matrix, and define
\begin{equation}
\Amat = \Pmat_x \Rmat \Pmat_y^\top, \qquad
\Bmat = \Pmat_y \Rmat^{-1} \Pmat_x^\top.
\end{equation}
Then $\Mmat = \Amat \Bmat = \Pmat_x \Rmat \Pmat_y^\top \Pmat_y \Rmat^{-1} \Pmat_x^\top = \Pmat_x \Pmat_x^\top = \Mmat^*$, and similarly $\Nmat = \Pmat_y \Pmat_y^\top = \Nmat^*$.
\end{proof}

\section{Additional results and ablations}
\label{app:results_ablation}

Note that due to the normalisation of some metrics in \cref{app:InterVench,app:InterVenchA}, some values can go above 1.

\begin{table*}[h]
  \centering
  \scalebox{0.9}{\begin{tabular}{lcccccc}
    \toprule
    Regularization & $\Mextract$ ($\uparrow$) & $\Mtr$ ($\uparrow$) & $\Mcont$ ($\uparrow$) \\
    \midrule
    $\Ldcy$              & 0.99  & 0.00  & 0.00 \\
    \midrule
    $\Ldcy + \Ltr$       & 0.99  & 0.99  & 0.00 \\
    $\Ldcy + \Lcont$     & 1.00  & 0.98  & 0.99 \\
    \midrule
    $\Ldcy + \Ltrdist$   & 0.99  & 0.99  & 0.00 \\
    $\Ldcy + \Lcontdist$ & 1.00  & 0.98  & 0.99 \\
    \bottomrule
  \end{tabular}}
  \caption{In the \textbf{synthetic DGP experiment}, the $\Pcont$ and $\Ptr$ duality \textbf{does not} hold. The instance-wise vs distributional duality \textbf{does} hold.}
  \label{tab:translation_vs_concept_consistency_synthetic}
\end{table*}

\begin{table*}[h]
  \centering
  \resizebox{\linewidth}{!}{
  \begin{tabular}{lCCCCCCCCc}
\toprule
& \multicolumn{9}{c}{\InterVench --- Language} \\
\cmidrule(lr){2-10}
& \multicolumn{5}{c}{Latent Space} & \multicolumn{3}{c}{Concept Space} & \multicolumn{1}{c}{} \\
\cmidrule(lr){2-6} \cmidrule(lr){7-9}
Regularization
& $R^2$ & Faith & SP & Unlearn & TPP &
                  SP & Unlearn & TPP & $\Mextract$
\\
\midrule
    $\Ldcy$              & 0.6243 & 0.9532 & 0.4663 & 0.9498 & 0.9827 & 0.4133 & 0.9075 & 0.5030  & \dcycolor{0.7250} \\
    \midrule
    $\Ldcy + \Ltr$      & 0.6746 & 0.9606 & 0.4631 & 0.9733 & 1.0026 & 0.3914 & 0.9897 & 0.9954  & \dcycolor{0.8063} \\
    $\Ldcy + \Ltrdist$   & 0.0003 & 0.1701 & 0.4645 & 0.9595 & 0.9893 & 0.0000 & 0.0000 & 0.0000  & \dcycolor{0.3230} \\
    \midrule
    $\Ldcy + \Lcont$     & 0.7298 & 0.9854 & 0.4589 & 0.9806 & 1.0061 & 0.1850 & 0.9513 & 1.3475  & \dcycolor{0.8306} \\
    $\Ldcy + \Lcontdist$ & 0.7166 & 0.9825 & 0.4683 & 0.9865 & 1.0048 & 0.2104 & 0.8391 & 0.9188  & \dcycolor{0.7659} \\
    \midrule
    $\Ldcy + \Ltr + \Lcont$              & 0.6983 & 0.9621 & 0.4570 & 0.9736 & 0.9984 & 0.3138 & 0.9749 & 1.1053  & \dcycolor{0.8104} \\
    $\Ldcy + \Ltrdist + \Lcontdist$      & 0.4412 & 0.9094 & 0.4265 & 0.9689 & 0.9950 & 0.2226 & 1.1094 & 1.3087  & \dcycolor{0.7977} \\
    mixed (1 in 10)                    & 0.7111 & 0.9868 & 0.4437 & 0.9840 & 1.0096 & 0.3593 & 0.9399 & 0.9942  & \dcycolor{0.8036} \\
    mixed (1 in 100)                    & 0.7078 & 0.9868 & 0.4521 & 0.9748 & 0.9990 & 0.3591 & 0.9276 & 0.8923  & \dcycolor{0.7874} \\
    mixed (1 in 1000)                    & 0.7120 & 0.9841 & 0.4507 & 0.9676 & 0.9905 & 0.3562 & 0.9903 & 1.0059  & \dcycolor{0.8072} \\
    \bottomrule
  \end{tabular}}
  \caption{Detailed \InterVench metrics on \textbf{language embeddings} under different regularization regimes.}
  \label{tab:intervench_ablation_studies_language}
\end{table*}

\begin{table*}[h]
  \centering
  \resizebox{\linewidth}{!}{
  \begin{tabular}{lCCCCCcCCCc}
\toprule
& \multicolumn{9}{c}{\InterVenchA --- Language} \\
\cmidrule(lr){2-10}
& \multicolumn{5}{c}{Translation} & \multicolumn{1}{c}{} & \multicolumn{3}{c}{Concept Consistency} & \multicolumn{1}{c}{} \\
\cmidrule(lr){2-6} \cmidrule(lr){8-10}
Regularization
& $R^2$ & Faith & SP & Unlearn & TPP & $\Mtr$ &
                  SP & Unlearn & TPP & $\Mcont$
\\
\midrule
    $\Ldcy$              & 0.0000 & 0.0000 & 0.0000 & 0.0000 & 0.0000  & \trcolor{0.0000} & 0.0000 & 0.0000 & 0.0000  & \contcolor{0.0000} \\
    \midrule
    $\Ldcy + \Ltr$      & 0.4897 & 0.9051 & 0.4348 & 0.7009 & 0.7718  & \trcolor{0.6605} & 0.1821 & 0.3131 & 0.4306  & \contcolor{0.3086} \\
    $\Ldcy + \Ltrdist$   & 0.0000 & 0.0000 & 0.1891 & 0.3409 & 0.3864  & \trcolor{0.1833} & 0.0000 & 0.0000 & 0.0000  & \contcolor{0.0000} \\
    \midrule
    $\Ldcy + \Lcont$     & 0.0000 & 0.0000 & 0.2233 & 0.4710 & 0.4660  & \trcolor{0.2321} & 0.2783 & 0.8061 & 0.9278  & \contcolor{0.6707} \\
    $\Ldcy + \Lcontdist$ & 0.0000 & 0.0000 & 0.0000 & 0.0000 & 0.0000  & \trcolor{0.0000} & 0.1177 & 0.3549 & 0.4103  & \contcolor{0.2943} \\
    \midrule
    $\Ldcy + \Ltr + \Lcont$              & 0.5863 & 0.9410 & 0.4150 & 0.8174 & 0.9075  & \trcolor{0.7334} & 0.2839 & 0.8740 & 0.9548  & \contcolor{0.7042} \\
    $\Ldcy + \Ltrdist + \Lcontdist$      & 0.0000 & 0.0000 & 0.1373 & 0.1278 & 0.1388  & \trcolor{0.0808} & 0.0000 & 0.0000 & 0.0000  & \contcolor{0.0000} \\
    mixed (1 in 10)                    & 0.6307 & 0.9528 & 0.3929 & 0.7120 & 0.7611  & \trcolor{0.6899} & 0.3569 & 0.9252 & 0.9615  & \contcolor{0.7479} \\
    mixed (1 in 100)                    & 0.6272 & 0.9546 & 0.4068 & 0.7366 & 0.7961  & \trcolor{0.7043} & 0.3548 & 0.9350 & 0.9411  & \contcolor{0.7436} \\
    mixed (1 in 1000)                    & 0.6310 & 0.9570 & 0.3895 & 0.8033 & 0.8425  & \trcolor{0.7247} & 0.3541 & 0.9136 & 0.9858  & \contcolor{0.7512} \\
    \bottomrule
  \end{tabular}}
  \caption{Detailed \InterVenchA metrics on \textbf{language embeddings} under different regularization regimes.}
  \label{tab:intervencha_ablation_studies_language}
\end{table*}

\begin{figure}
    \centering
    \includegraphics[width=0.49\linewidth]{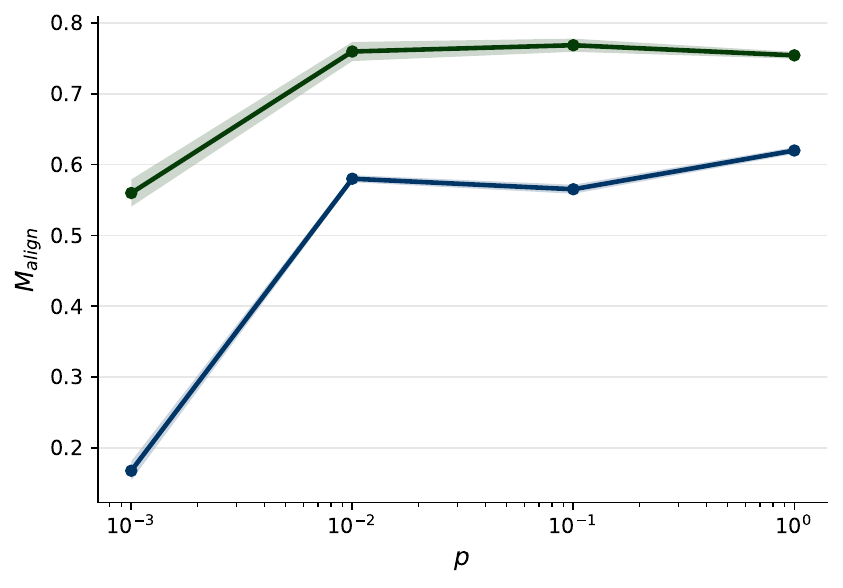}
    \hfill
    \includegraphics[width=0.49\linewidth]{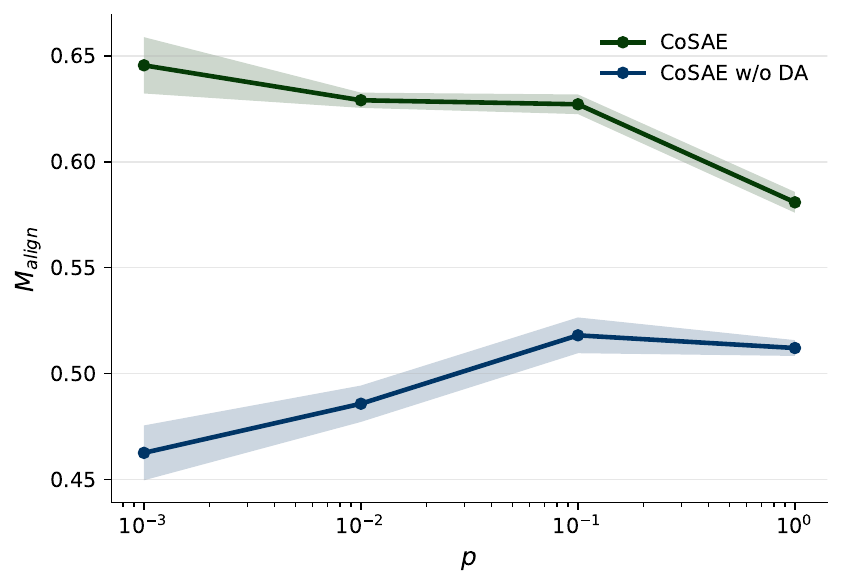}
    \caption{Ablation of the effect of distributional alignment when computing element-wise alignment losses on a proportion $p$ of the training data---used as anchor points. CoSAE is the mixed regime described in \cref{sec:anchoring}. ``CoSAE w/o DA'' is the same procedure but without Distributional Alignment regularization. This plot shows results for SAEs trained on visual (\emph{left}) and textual (\emph{right}) embeddings. Plain dots report the mean value over 10 runs, shaded regions indicate $\pm \mathrm{std}$.}
    \label{fig:u2vsu3}
\end{figure}

\begin{table*}[h]
  \centering
  \resizebox{\linewidth}{!}{
  \begin{tabular}{lCCCCCCCCc}
\toprule
& \multicolumn{9}{c}{\InterVench --- Vision} \\
\cmidrule(lr){2-10}
& \multicolumn{5}{c}{Latent Space} & \multicolumn{3}{c}{Concept Space} & \multicolumn{1}{c}{} \\
\cmidrule(lr){2-6} \cmidrule(lr){7-9}
Regularization
& $R^2$ & Faith & SP & Unlearn & TPP &
                  SP & Unlearn & TPP & $\Mextract$
\\
\midrule
    $\Ldcy$              & 0.8385 & 1.0707 & 0.7788 & 0.9438 & 0.9820 & 0.7572 & 0.9442 & 0.9872  & \dcycolor{0.9128} \\
    \midrule
    $\Ldcy + \Ltr$      & 0.7228 & 1.0896 & 0.7715 & 0.9237 & 0.9945 & 0.7772 & 0.9475 & 0.9859  & \dcycolor{0.9016} \\
    $\Ldcy + \Ltrdist$   & 0.8222 & 1.0426 & 0.7282 & 0.9570 & 0.9931 & 0.6843 & 0.9116 & 0.9751 & \dcycolor{0.8893} \\
    \midrule
    $\Ldcy + \Lcont$     & 0.8312 & 1.0598 & 0.7799 & 0.9417 & 0.9954 & 0.0000 & 0.0000 & 0.0000  & \dcycolor{0.5760} \\
    $\Ldcy + \Lcontdist$ & 0.8228 & 1.0605 & 0.7776 & 0.9095 & 0.9932 & 0.7641 & 0.9185 & 0.9921  & \dcycolor{0.9048} \\
    \midrule
    $\Ldcy + \Ltr + \Lcont$              & 0.7160 & 1.0945 & 0.7721 & 0.9793 & 0.9964 & 0.3230 & 0.8885 & 0.9179  & \dcycolor{0.8360} \\
    $\Ldcy + \Ltrdist + \Lcontdist$      & 0.7352 & 1.0582 & 0.7737 & 0.9459 & 0.9853 & 0.7604 & 0.8935 & 0.9889  & \dcycolor{0.8926} \\
    mixed (1 in 10)                    & 0.7124 & 1.1245 & 0.7623 & 0.9615 & 0.9951 & 0.7630 & 0.9555 & 1.0011  & \dcycolor{0.9094} \\
    mixed (1 in 100)                    & 0.6734 & 1.0819 & 0.7675 & 0.9385 & 0.9944 & 0.7541 & 0.9549 & 0.9815  & \dcycolor{0.8933} \\
    mixed (1 in 1000)                    & 0.6971  & 1.0931  & 0.7698 & 0.9460  & 0.9988 & 0.7544 & 0.9532 & 0.9901  & \dcycolor{0.9003} \\
    \bottomrule
  \end{tabular}}
  \caption{Detailed \InterVench metrics on \textbf{vision embeddings} under different regularization regimes.}
  \label{tab:intervench_ablation_studies_vision}
\end{table*}

\begin{table*}[h]
  \centering
  \resizebox{\linewidth}{!}{
  \begin{tabular}{lCCCCCcCCCc}
\toprule
& \multicolumn{10}{c}{\InterVenchA --- Vision} \\
\cmidrule(lr){2-11}
& \multicolumn{5}{c}{Translation} & \multicolumn{1}{c}{} & \multicolumn{3}{c}{Concept Consistency} & \multicolumn{1}{c}{} \\
\cmidrule(lr){2-6} \cmidrule(lr){8-10}
Regularization
& $R^2$ & Faith & SP & Unlearn & TPP & $\Mtr$ &
                  SP & Unlearn & TPP & $\Mcont$
\\
\midrule
    $\Ldcy$              & 0.0000 & 0.6862 & 0.0000 & 0.0000 & 0.0000  & \trcolor{0.1372} & 0.0000 & 0.0000 & 0.0000  & \contcolor{0.0000} \\
    \midrule
    $\Ldcy + \Ltr$      & 0.4660 & 1.1853 & 0.5350 & 0.8804 & 0.5750  & \trcolor{0.7283} & 0.1896 & 0.2415 & 0.2211  & \contcolor{0.2174} \\
    $\Ldcy + \Ltrdist$   & 0.0000 & 0.7813 & 0.2952 & 0.8119 & 0.7327  & \trcolor{0.5242} & 0.0000 & 0.0000 & 0.0000  & \contcolor{0.0000} \\
    \midrule
    $\Ldcy + \Lcont$     & 0.0000 & 0.6998 & 0.0000 & 0.0000 & 0.0000  & \trcolor{0.1400} & 0.4598 & 0.8263 & 0.5209  & \contcolor{0.6023} \\
    $\Ldcy + \Lcontdist$ & 0.0000 & 0.7010 & 0.0000 & 0.0000 & 0.0000  & \trcolor{0.1402} & 0.0000 & 0.0000 & 0.0000  & \contcolor{0.0000} \\
    \midrule
    $\Ldcy + \Ltr + \Lcont$       & 0.4705 & 1.1824 & 0.5451 & 0.9235 & 0.5440  & \trcolor{0.7331} & 0.4694 & 0.8765 & 0.5102  & \contcolor{0.6187} \\
    $\Ldcy + \Ltrdist + \Lcontdist$      & 0.0000 & 0.7141 & 0.0000 & 0.0000 & 0.0000  & \trcolor{0.1428} & 0.0000 & 0.0000 & 0.0000  & \contcolor{0.0000} \\
    mixed (1 in 10)                    & 0.5243 & 1.1509 & 0.5221 & 0.9225 & 0.5228  & \trcolor{0.7285} & 0.7267 & 0.9349 & 0.7786  & \contcolor{0.8134} \\
    mixed (1 in 100)                    & 0.4224 & 1.0931 & 0.6267 & 0.8847 & 0.6318  & \trcolor{0.7317} & 0.7152 & 0.9312 & 0.7373  & \contcolor{0.7946} \\
    mixed (1 in 1000)                    & 0.4056  & 0.9455  & 0.4667  & 0.6016 & 0.5905 & \trcolor{0.6020} & 0.3528 & 0.8211 & 0.4584  & \contcolor{0.5442} \\
    \bottomrule
  \end{tabular}}
  \caption{Detailed \InterVenchA metrics on \textbf{vision embeddings} under different regularization regimes.}
  \label{tab:intervencha_ablation_studies_vision_app}
\end{table*}

\begin{table*}[h]
  \centering
  \scalebox{0.9}{\begin{tabular}{lccc}
    \toprule
    Regularization & $\Mextract$ ($\uparrow$) & $\Mtr$ ($\uparrow$) & $\Mcont$ ($\uparrow$) \\
    \midrule
    $\Ldcy$              & \best{0.8956}  & 0.0000  & 0.0000 \\
    \midrule
    $\Ldcy + \Ltr$       & 0.7391  & 0.2974  & 0.2377 \\
    $\Ldcy + \Lcont$     & 0.8786  & 0.0000  & 0.3271 \\
    \midrule
    $\Ldcy + \Ltrdist$   & 0.7658  & 0.000  & 0.0000 \\
    $\Ldcy + \Lcontdist$ & 0.8879  & 0.000  & 0.0000 \\
    \midrule
    $\Ldcy + \Ltr + \Lcont$              & 0.7496  & 0.2765  & 0.3454 \\
    $\Ldcy + \Ltrdist + \Lcontdist$      & 0.8889  & 0.000  & 0.0000 \\
    mixed (1 in 1000)                    & 0.7814  & \best{0.5053}  & \best{0.5535} \\
    \bottomrule
  \end{tabular}}
  \caption{\textbf{Ablation studies on the effect of regularization terms.} SAE quality on \textbf{multimodal embeddings} under different regularization regimes.}
  \label{tab:ablation_studies_multimodal}
\end{table*}

\end{document}